\documentclass[Afour,times,sageh]{sagej}

\usepackage{moreverb,url}
\usepackage[normalem]{ulem} 

\usepackage[colorlinks,bookmarksopen,bookmarksnumbered,citecolor=red,urlcolor=red]{hyperref}

\newcommand\BibTeX{{\rmfamily B\kern-.05em \textsc{i\kern-.025em b}\kern-.08em
T\kern-.1667em\lower.7ex\hbox{E}\kern-.125emX}}

\def\volumeyear{2026}

\usepackage{xcolor}

\definecolor{b}{RGB}{0, 0, 255}
\usepackage{placeins}
  \usepackage{blindtext}
\usepackage[english]{babel}
\usepackage{float}
\usepackage{mathtools}
\usepackage{hyperref}
\usepackage{nicefrac}
\usepackage{amsthm}
\usepackage{amsmath} 
\usepackage[Symbol]{upgreek}
\usepackage{amssymb}

\usepackage{bm}
\newcommand{\vct}{\mathbf}
\newcommand{\vect}[1]{\boldsymbol{#1}}
\newtheorem{theorem}{Theorem}

\newtheorem{remark}{Remark}
\usepackage{multirow}
\usepackage{graphicx}
\usepackage{color}

\usepackage[caption=false,font=footnotesize]{subfig}
\usepackage{changepage}
\usepackage{enumitem}
\usepackage{siunitx}

\begin{document}

\title{\LARGE \bf
Postural Virtual Fixtures for Ergonomic Physical Interactions with Supernumerary Robotic Bodies
}

\author{Theodora Kastritsi,  Marta Lagomarsino, 
 and Arash Ajoudani}  
\affiliation{Authors are with Human-Robot Interfaces and Interaction Laboratory, Istituto Italiano di Tecnologia, Genoa, Italy. }

\corrauth{Theodora Kastritsi, \email{tkastrit@gmail.com }}

\begin{abstract}

Conjoined collaborative robots, functioning as supernumerary robotic bodies (SRBs), can enhance human load tolerance abilities. However, in tasks involving physical interaction with humans, users may still adopt awkward, non-ergonomic postures, which can lead to discomfort or injury over time. In this paper, we propose a novel control framework that provides kinesthetic feedback to SRB users when a non-ergonomic posture is detected, offering resistance to discourage such behaviors. This approach aims to foster long-term learning of ergonomic habits and promote proper posture during physical interactions.  To achieve this, a virtual fixture method is developed, integrated with a continuous, online ergonomic posture assessment framework. Additionally, to improve coordination between the operator and the SRB, which consists of a robotic arm mounted on a floating base, the position of the floating base is adjusted as needed. Experimental results demonstrate the functionality and efficacy of the ergonomics-driven control framework, including two user studies involving practical loco-manipulation tasks with 14 subjects, comparing the proposed framework with a baseline control framework that does not account for human ergonomics.
\end{abstract}
\keywords{
Human Performance Augmentation, 
Physical Human-Robot Interaction,
Compliance and Impedance Control,
Whole-Body Motion Planning and Control,
Human Factors and Human-in-the-loop 
}

\maketitle

\section{INTRODUCTION}

Technological advancements have enabled the augmentation of human capabilities through robotics, transforming what was once a futuristic vision into reality. Exoskeletons, for instance, have shown significant promise in enhancing human strength, allowing users to perform physically demanding tasks more efficiently. However, some challenges remain, such the additional weight that strains joints and muscles \citep{de2016exoskeletons, bar2021influence}, operational pressures that can lead to discomfort or tissue irritation \citep{kermavnar2021effects}, and limitations in natural movement, mobility, and postural balance \citep{kermavnar2021effects, theurel2019occupational}.

Another significant innovation in human augmentation is the development of supernumerary robotic limbs (SRLs), which are wearable robotics devices designed to extend human capabilities. SRLs by providing additional robotic limbs, such as extra arms, legs, or fingers, can also compensate for lost motor functions in individuals with physical impairments \citep{hendriks2024enhancing}. SRLs hold substantial potential, as they can be attached at various points on the body and adopt diverse structural configurations, enabling them to support the user and perform a wide range of auxiliary functions.  To achieve this, they require an interface with the human user, which can be implemented in multiple ways; for example, it may include motion tracking, EEG-based control, or EMG-based control \citep{eden2022principles,prattichizzo2021human}. Recent work has also demonstrated multimodal control of an extra robotic arm using gaze and respiration signals \citep{dominijanni2023human}. Despite their advantages, SRLs share several limitations with exoskeletons, including weight, discomfort, and reduced mobility \citep{tong2021review}, that limit their practical usability. Like exoskeletons, they rely heavily on human coupling for operation, often redistributing the load from one limb to another rather than fully alleviating it. In particular, in setups that incorporate large robotic limbs, the user must support not only the weight of the external loads but also the weight of the device itself. Although recent efforts have been made to develop lightweight systems \citep{hasanen2024design}, the limited durability of their components and low payload capacity continue to restrict their practical usability. Additionally, SRLs face unique challenges, such as compensating for interference from the wearer's body motion to achieve robot-like precision \citep{tong2021review}.
Today, essential aspects such as wearability, efficiency, and usability remain unresolved for these technologies \citep{yang2021supernumerary}.

To overcome these limitations, supernumerary robotic bodies (SRBs)  have recently emerged as a promising solution for physically assisting and augmenting humans in conjoined loco-manipulation tasks \citep{kim2020moca, giammarino2024super}. Unlike exoskeletons and SRLs, which are often constrained by their reliance on the human body and carry significant weight, SRBs are floating-base robotic manipulators that, like SRLs, can adopt various structural configurations  but uniquely ground external loads, eliminating extra weight on the user and ensuring comfort and mobility. SRBs, like SRLs, require a human interface, which in the case of SRBs can be attached to or detached from the robot, enabling local control and teleoperation, respectively \citep{raei2024multipurpose}. In particular, the local control is equipped with an admittance-type interface that provides an intuitive modality for physical human-robot interaction (pHRI), allowing the users to directly command the robot's motion. By decoupling the human-exerted force from the environmental force, SRBs enable the robot to handle heavy tools while allowing the user to guide the end-effector effortlessly. 
Furthermore, when not  operated by humans, SRBs can  function autonomously, reducing equipment downtime and associated operational costs. While SRBs offer significant potential for enhancing task efficiency and reducing physical demands, improper usage during pHRI can still pose long-term health risks to workers. This is especially true in scenarios where human operators adopt and sustain awkward, non-ergonomic postures while interacting with the SRBs. Even in low-load situations, poor posture can result in increased musculoskeletal strain, reduced circulation, joint misalignment, and other related issues, ultimately compromising worker well-being over time \citep{lorenzini2023ergonomic}. This remaining issue can undermine the ergonomic benefits of SRBs, hence, advanced control strategies and online monitoring systems are essential to exploit their full potential. By integrating adaptive ergonomics monitoring and intelligent feedback mechanisms, SRBs can dynamically adjust their behavior to support the operator's posture and minimize the risk of strain, thereby enhancing both safety and efficiency in collaborative work environments. 

\subsection{Related Works }
\subsubsection{Postural Ergonomics:} Research on ergonomic posture assessment, such as in \cite{mcatamney1993rula,hignett2000rapid}, provides valuable insights into human body ergonomics through joint-angle-based evaluations.   These assessments have been used to evaluate user comfort in human-robot interaction tasks offline \citep{morfino2024hybrid} and have also contributed to the development of robotic strategies that support ergonomic collaboration, as reviewed in \cite{proia2021control, lorenzini2023ergonomic}. In addition, online ergonomic assessments based on joint range of motion have been proposed in \cite{lorenzini2022online, lagomarsino2023maximising}, assuming that the optimal joint position lies at the midpoint of its range, which may not hold true for all joints. Most existing strategies supporting ergonomic collaboration primarily focus on the planning level, aiming to position objects held by robots in ergonomically favorable poses for users \citep{shafti2019real, zanchettin2019collaborative}, without considering pHRI, which limits robots' ability to dynamically respond to user-applied forces. To address this, \cite{ferraguti2020unified} introduced a framework that combines ergonomic positioning with admittance control. However, this approach does not account for user ergonomics during manipulation, allowing non-ergonomic postures to emerge despite an initially ergonomic setup. In a recent work, \cite{liao2023ergo} utilizes an SRB for the co-transportation of an object and an admittance control model designed to guide the user's wrist toward a more ergonomic position via a virtual force. However, this approach only considers motion in the sagittal plane for generating the virtual force and lacks passivity analysis, which is crucial for ensuring stable pHRI \citep{keemink2018admittance}. Hence, the remaining challenge is to develop a method that can be integrated into the SRB control framework to effectively promote and maintain good postures during pHRI.
\subsubsection{Virtual Fixture:}
To provide kinesthetic feedback and assist human operators in scenarios with partial task knowledge, the concept of Virtual Fixture (VF) has been widely explored in the field of pHRI, as well as in other areas such as teleoperation, where it was first proposed \citep{rosenberg1992use}. Virtual Fixtures (VFs) are implemented to influence robot motion by guiding it within predefined desired regions, called Guidance VFs, or preventing it from entering forbidden or sensitive areas, known as Forbidden-Region VFs. A review of VFs can be found in \cite{bowyer2013active}. Various approaches to enforcing VFs exist; some use energy storage techniques, such as artificial potential fields \citep{kastritsi2019stability} and the proxy/god-object approach \citep{ryden2011proxy}, while others avoid energy storage, incorporating friction models in impedance-controlled robots \citep{bowyer2015dissipative} or applying anisotropic projections of human-exerted forces in an admittance model \citep{castillo2010virtual}. However, the latter category does not provide kinesthetic feedback when the robot end-effector is stationary unless combined with an energy storage method. VFs are often represented using point clouds \citep{kastritsi2019stability}, polygonal meshes \citep{zilles1995constraint, ruspini1997haptic}, parametric curves or surfaces \citep{papageorgiou2020passive2}, or Cartesian paths learned through demonstration \citep{papageorgiou2020passive}. However, these representations are typically defined in the robot's joint or task space, making them unsuitable for enforcing ergonomics-based virtual fixtures, as posture ergonomics is inherently human-centric and mapping between robot and human joint configurations is non-trivial.  A recent work \citep{hu2024proxy} that utilizes Guidance VF claims that by constraining the rotation of the robot end-effector, they achieve a more ergonomic posture for the user. However, this claim is not generally true and lacks support from both theoretical reasoning and experimental validation.

\subsubsection{Mobile Manipulators}: Mobile manipulators, which combine a robotic arm with a mobile platform, offer enhanced mobility and adaptability in dynamic environments. However, they have not been widely adopted for physical human-robot collaboration. In recent years, advancements in sensor technology have facilitated their use in collaborative tasks with humans, such as hand-guiding \citep{navarro2017framework}, co-transportation \citep{benzi2022whole, sirintuna2024enhancing}, and assisted guidance \citep{balatti2024robot}. Despite this progress, existing research has often overlooked the consideration of human ergonomics during collaboration, an important factor that can significantly impact user comfort, safety, and long-term effectiveness during physical interaction. Utilizing a mobile manipulator in the form of an SRB, \cite{liao2023ergo} proposed an admittance control schema to guide the user's wrist ergonomically; however, it is limited to the sagittal plane and lacks passivity analysis. Furthermore, the framework overlooks the position of the human user, which can constrain their motion, hindering the task.

\subsection{Contributions} 
\begin{figure*}[!ht]
     \centering \includegraphics[width=0.98\textwidth]{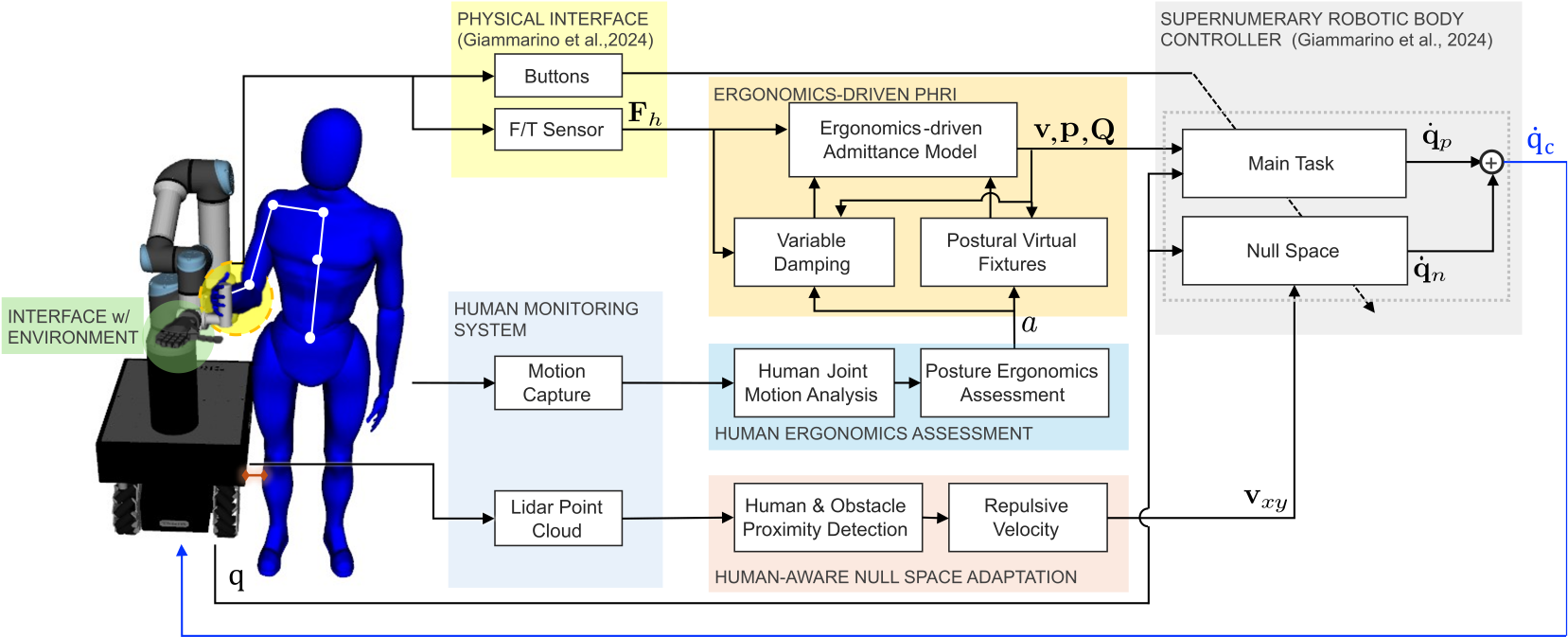}     
\caption{Block diagram illustrating the proposed methodology, outlining the key components and their interactions.}
\label{fig:block-diagram}
\end{figure*}

Motivated by the aforementioned challenges and gaps, this work proposes a novel ergonomics-driven control framework for SRBs based on online human posture assessment and kinesthetic feedback. By providing kinesthetic feedback when non-ergonomic postures are detected, while dynamically adjusting the position of the floating base to optimize coordination between the user and the SRB, the proposed framework fosters a seamless and ergonomic pHRI. The key contributions of the proposed framework include:
\begin{itemize}
    \item Ergonomics Postural Virtual Fixtures: For the first time, VF are designed with consideration for the user’s postural ergonomics, helping to prevent poor ergonomics while preserving flexibility and autonomy in task execution.
    \item Passivity Assurance:  Through rigorous mathematical analysis, it is ensured that the desired dynamic behavior of the robot end-effector, combined with the Ergonomics Postural Virtual Fixtures, provides a passive framework. Passivity is not inherently guaranteed when using a VF approach and is particularly critical during physical human-robot interaction, as it ensures the safety of the user.
    \item Online Continuous RULA-based Ergonomics Assessment Metric: Based on the Rapid Upper Limb Assessment (RULA) of \cite{mcatamney1993rula}, a metric is proposed and formulated as a continuous smooth function. Unlike conventional ergonomics metrics, which are often discontinuous or not designed with widely recognized ergonomics assessment tools in mind, the proposed method is specifically designed for seamless integration into robotic control frameworks and reflects a reliable ergonomic level of posture.
    \item Null-space coordination: The redundancy of SRBs is exploited to ensure that the user’s motion is not obstructed by repelling the SRBs floating base when the human is in close proximity, while simultaneously preventing the robot from colliding with obstacles in the environment.
 
\end{itemize}
In summary, this novel approach offers a robust solution for ergonomic-aware physical human-robot interaction, enhancing user safety and adaptability by reducing human physical load and avoiding non-ergonomic postures at the same time. This control framework surpasses a simple aggregation of a whole-body SRB controller, VF enforcement, null-space adaptation, and online human posture ergonomic assessment, making substantial contributions to the integration of these components and enhancing each individual aspect of the system.  Two extensive user studies were conducted, demonstrating for the first time that Ergonomics Postural Virtual Fixtures help users adopt more ergonomic postures during conjoint loco-manipulation tasks, in addition to completely offloading the payloads at hand through the use of the proposed SRB, while also exhibiting a learning effect.

The rest of the paper is organized as follows. The \textit{PROPOSED METHODOLOGY} section presents the ergonomics-aware control framework illustrated by the block diagram in Figure \ref{fig:block-diagram}. Subsequently, the \textit{EXPERIMENTS} section outlines the implementation of the proposed solution, aimed at validating its effectiveness. This is followed by the \textit{EXPERIMENTAL RESULTS} section, which presents and analyzes the findings. Finally, the \textit{CONCLUSIONS} section summarizes the results and highlights potential future research directions. 

\section{PROPOSED METHODOLOGY }

\subsection{Postural Virtual Fixtures}
\label{sec:erg_adm}
In this subsection, the ergonomics-driven admittance model of the robot end-effector is proposed, introducing the motion dynamics of the postural virtual fixtures and providing the passivity analysis. One of the control objectives is to kinesthetically inform the user whenever a non-ergonomic posture is detected. To achieve this, the following admittance model is used to enable the robot to respond appropriately to the human generalized force $\mathbf{F}_{h}=[\mathbf f_h^\mathrm{T}  \  \vect \tau_h^\mathrm{T}]^\mathrm{T} \in \mathbb{R} ^6$  exerted  on its end-effector:
\begin{equation} \label{eq:adm_x2y}
\mathbf{M}_{d} \dot{\mathbf{v}} + \mathbf{D}_{d} \mathbf{v} = \mathbf{F}_{h}+ \mathbf{u}_c
\end{equation}
where $\mathbf{v}= [ \dot{\mathbf{p}}\ \;\boldsymbol{\omega}]^\mathrm{T} \in \mathbb{R}^6 $ is the robot end-effector generalised  velocity, $\mathbf{M}_{d} \in \mathbb{R} ^{6 \times 6}$  is a positive constant diagonal matrix of the target inertia$,  \ \mathbf{D}_{d} \in \mathbb{R} ^{6 \times 6}$ is a positive  matrix of
the target damping,  and $\mathbf{u}_c\in \mathbb{R} ^6$ is a control term designed to provide kinesthetic feedback, encouraging the human operator to maintain an ergonomic posture.

In order to design the control term $\mathbf{u}_c$, we introduce the concept of postural virtual fixtures, leveraging the idea of the god-object algorithm. Originally developed for haptic rendering in virtual environments, the god-object algorithm, also referred to as the proxy algorithm, allows the god-object to follow the robot end-effector in free space while remaining outside virtual obstacles defined in Cartesian space when contact between the obstacles and the god-object occurs. Then the deviation between the god-object and the robot's position is used to the robot's controller to generate haptic feedback, enabling realistic rendering of interactions. In contrast to the majority of existing works, here we express the motion of the god-object using a motion dynamics equation, rather than using an algorithmic approach, offering the possibility of providing passivity analysis in a closed form. Furthermore, the god-object here serves as an idealized representation of the robot end-effector pose, constrained to the space where the human posture is not at high risk, thus constructing the virtual fixture with consideration of the human side rather than the task. With this in mind, the motion dynamics of the god-object is expressed as follows:
\begin{equation}
{\mathbf{v}}_{g}=-\mathrm k_r f(a,\mathbf p_e)\begin{bmatrix}
   \mathbf p_e\\ {\vect{\epsilon}_e}
  \end{bmatrix}
  \label{eq:dyn_god}
\end{equation}
where ${\mathbf{v}}_{g}=[ \dot{\mathbf p}_g\ {\vect{\omega}_g}]^\mathrm{T} \in \mathbb{R}^{6}$ is the god-object generalized velocity,  $\mathrm{k_r}\in \mathbb{R}_>$ is a positive definite gain that, together with the value of the $f(a,\mathbf p_e)$, affects the time of convergence,   $\mathbf p_e,\ \vect \epsilon_e \in \mathbb{R}^{3}$  are the position and orientation deviations, respectively and $f(a, \mathbf{p}_e)\in [0 \ \  1 ]$ is a continuous variable that depends on the ergonomics of the human posture $a \in [0 \ 1 ]$  and on the position deviation. 

The position deviation is defined as   $\mathbf p_e=\mathbf p_g-\mathbf p$, representing the difference between the robot end-effector and the god-object.
   The orientation deviation  in quaternion terms, associated with the deviation rotation matrix ${\mathbf R}_e={\mathbf R}_{g}{\mathbf R}^\mathrm{T} \in SO(3)$, is given by $\mathbf{Q}_e\triangleq\mathbf{Q}_g\otimes \mathbf{Q}^{-1}=\begin{bmatrix} \eta_e \\ \vect{\epsilon}_e\end{bmatrix} \in \mathbb{S}^3$ where  $\eta_e = \cos\left(\frac{\phi_e}{2}\right)\in \mathbb{R}_{\geq}$,  with $\phi_e \in (-\pi, \pi)$ being the deviation angle in the equivalent angle-axis representation. Here, $\mathbf{Q}^{-1} \in \mathbb{S}^3 $  denotes the inverse of the orientation quaternion $\mathbf{Q} \in \mathbb{S}^3$, and $\otimes$ denotes quaternion multiplication.  Notice that to calculate the orientation in quaternion terms   $\mathbf{Q}$ and  $\mathbf{Q}_g$, we integrate the angular velocities, $\boldsymbol{\omega}$ and $\boldsymbol{\omega}_g$ provided by \eqref{eq:adm_x2y} and \eqref{eq:dyn_god}, using the exponential map  $\mathrm{exp}_Q: \, \mathbb{R}^3\rightarrow\mathbb{S}^3 $, as follows:
 \begin{equation} \label{eq:int_qr}
         \mathbf Q_i(t+\mathrm{T}_c)=\mathrm{exp}_Q{(\frac{1}{2} {\vect \omega}_i\mathrm{T}_c)}\, \otimes \, \mathbf Q_i(t) \ \in \mathbb{S}^3 
    \end{equation}
where $\mathrm{T}_c\in \mathbb{R}_>$ is the robot control cycle.

The function $f(a, \mathbf{p}_e)$ is non-negative, well-defined, which modulates the impact of deviations on the motion of the god-object. Our goal is for the god-object to avoid following the robot's motion when the human is in a high-risk non-ergonomic posture, i.e.,  when $a=0$, while following it when the human is or returns to an ergonomic posture with a velocity based on the value of $a$. To achieve this, we define a unit vector $\mathbf n \in \mathbb{R}^3 $ to reflect the robot's motion vector when the motion is ergonomic,  or to indicate the direction that caused the user to move away from an ergonomic posture when the human's posture is non-ergonomic. To achieve this, the motion dynamics of the vector $\mathbf n$ is designed  as follows:
\begin{equation} \label{eq:dotn}
        \dot {\mathbf n}=-{ \omega}_n\mathbf n\times ( \mathbf n \times  \mathbf n_d) , \ \ \
    {\omega}_n = \mathrm k_n  \, a \,{sin(|\phi_n|/2)}
\end{equation}
\begin{equation}\label{dire} \mathbf {n}_d=\begin{cases}
\dfrac{\dot {\mathbf p}}{||\dot {\mathbf p}||} \ \  \ \ \ , \ \text{if} \ ||\dot {\mathbf p}||> \epsilon  \\
 \ \ \ \mathbf 0_{3} \ \  \ \  \  \ ,  \  \text{otherwise}.
\end{cases} \in \mathbb{R}^3\end{equation}
where $\times$ denotes the vector cross product, $\mathrm k_n\in \mathbb{R}_>$ is a  gain term, and $\phi_n=ang(\mathbf n , \mathbf n_d) \ \in [-\pi, \  \pi)$ gives the angle between the current desired  direction of motion $\mathbf n_d$ and the vector  $\mathbf n$ as follows:
\begin{equation} \label{eq: angle}
ang(\mathbf n , \mathbf n_d)= \begin{cases} \  \  \  \ \  \ \     \ 0 \ \ \  \  ,    \text{ if } \mathbf n^\mathrm{T}\mathbf n_d=|| { \mathbf n \times \mathbf n_d }||=0 
   \\ atan2\big(|| { \mathbf n \times \mathbf n_d }||, { \mathbf n^\mathrm{T}\mathbf n_d }\big),    \text{ otherwise}.
  \end{cases}
\end{equation}
To calculate the vector $\mathbf n$, we integrate \eqref{eq:dotn}   using the exponential map  $\mathrm{exp}_R: \mathbb{R}^3 \rightarrow SO(3)$, as follows: 
 \begin{equation} \label{eq:n-new}
         \mathbf n(t+\mathrm{T}_c)=\mathrm{exp}_R{( {\omega}_n\mathrm{T}_c \,  \mathbf n\times  \mathbf n_d )}\, \mathbf n(t).
    \end{equation}
 \begin{figure}[!t]
     \centering \includegraphics[width=0.4\textwidth]{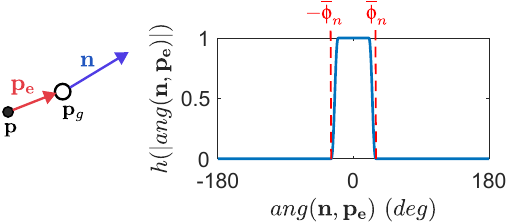}
        \caption{Visualization of the  sigmoid function $h$ used for 
for smooth switching.}
         \label{fig:h}
\end{figure}
In the case where the human is in a non-ergonomic posture, the vector $\mathbf n$ must form an acute angle with the position deviation  $\mathbf p_e$ to indicate that the human is moving back to the ergonomic space.  Based on this principle, we define the function  $\mathbf f(a,\mathbf p_e)$ to be zero when the user is in a non-ergonomic posture and moving further away from it. Conversely, it should be non-zero when the user is either in an ergonomic posture or attempting to return to it. The function that achieve these  is defined as follows:
\begin{equation} \label{eq:f}
    \mathbf f(a,\mathbf p_e)={a}^m+(1-{a}^m)h(|ang(\mathbf n , \mathbf{p_e})| ;\overline{\upphi}_{n}-\updelta_n,\overline{\upphi}_{n})
\end{equation}
where $m\in\mathbb{R}_>$ controls the sensitivity of the function $f(a,\mathbf p_e)$ to changes in $a$, with $m>1$ amplifying the effect of 
$a$, $m=1$ producing a linear response, and 
$m<1$ softening the transformation for smoother behavior. The function
$h(.)$ $\in[0  \ 1]$ is given by the following sigmoid function for smooth switching:
\begin{equation}\
    h(w;\underline{\mathrm n},\overline{\mathrm n})\hspace{-0.08cm} =\hspace{-0.08cm}\begin{cases}
		\hspace{-0.08cm}	1 , & \hspace{-0.3cm} \text{if $w<\underline{\mathrm n}$ }\\
           \hspace{-0.08cm} 0, & \hspace{-0.3cm} \text{if $w>\overline{\mathrm n}$ }\\ \hspace{-0.08cm}
   1-6c(w)^5+15c(w)^4-10c(w)^3,& \hspace{-0.3cm}\text{otherwise}   
		 \end{cases}
   \label{eq:h}
\end{equation}
with $c(w)=\dfrac{w-\underline{\mathrm n}}{\overline{\mathrm n}-\underline{\mathrm n}} $. A visualization of this function can be seen in Figure \ref{fig:h}. $\overline{\upphi}_{n}$ represents the angle below which the human is considered to be returning to the ergonomic space and $\updelta_n$ defines the range of angles where the sigmoid function $h$ transitions from 1 to 0.

To enhance the maintenance of an ergonomic posture in a physical human-robot interaction framework, we introduce a stiffness term through the robot control term $\mathbf u_c$  between the robot's  pose, determined by integrating the target admittance model \eqref{eq:adm_x2y} and the god-object pose, defined by integrating \eqref{eq:dyn_god}, as follow:
\begin{equation}\label{eq:u_c}
    \mathbf u_c=\mathbf K_d  \begin{bmatrix}
   \mathbf p_e\\ \mathbf R_e^\mathrm{T}{\vect{\epsilon}_e}
  \end{bmatrix}
\end{equation} 
where $\mathbf K_d =\{\mathrm k_p \mathbf I_{3 \times 3}, \mathrm k_o \mathbf I_{3 \times 3}\}\in \mathbb{R} ^{6 \times 6}$ is a positive definite constant stiffness
matrix. This control term provides kinesthetic feedback to the user about the ergonomics of their posture.
    
\begin{remark}
 Notice that the control term $\mathbf{u}_c$ allows the user to perceive deviations between the robot end-effector and the god-object pose. These deviations become significant when the user adopts a high-risk non-ergonomic posture and moves away from the god-object pose, i.e., when $f(a, \mathbf{p}_e) = 0$. However, when $f(a, \mathbf{p}_e) = 1$, the god-object can precisely follow the robot end-effector pose. In this case, the user can perceive the dynamics of the models \eqref{eq:adm_x2y} and \eqref{eq:dyn_god}; the parameters of these models should be selected appropriately to ensure the user experiences an enhanced sense of control during the task.
\end{remark}

\begin{theorem}
For the system \eqref{eq:adm_x2y}, \eqref{eq:dyn_god}, utilizing the control term  \eqref{eq:u_c}, the following statements are
valid under the exertion of the human force $\mathbf F_h$:
\begin{enumerate}[label=\arabic*)]
    \item the generalized robot velocity  $\mathbf v$ and human force  $\mathbf F_h$  form a passive pair.
    \item the pose deviations  between the robot and the god-object is bounded.
\end{enumerate}
\end{theorem}

\begin{proof}
1) Let us define  the following storage function:
\begin{equation} \label{eq:V5} V=\mathbf v^\mathrm{T}\dfrac{ \mathbf M_{d}}{2}\mathbf v+ \mathbf p_e^\mathrm{T} \dfrac{ \mathrm k_{p}}{2} \mathbf p_e+2\mathrm{k_o}\Psi_e\geq 0 \end{equation}
where   $\Psi_e = 1 - \eta_e \in \mathbb{R}_{\geq}$.  Notice that, since $\eta_e = \cos\left(\frac{\phi_e}{2}\right)$, $\Psi_e$ is non negative and equal to zero if and only if $\phi_e = 0$, for all $\phi_e \in (-\pi, \pi)$.

Taking the time derivative of 
\eqref{eq:V5} and substituting $\dot{\mathbf v}$ from \eqref{eq:adm_x2y} yields:
\begin{equation} \label{eq:dotV12} 
\begin{split}
\dot V =&-\mathbf v^\mathrm{T}\mathbf D_{d} \mathbf v+ \mathrm k_{p}\dot{\mathbf p}^\mathrm{T}\mathbf p_e+\mathrm k_{o}{\vect \omega}^\mathrm{T} \mathbf R_e ^\mathrm{T} \vect{\epsilon}_e\\ &+\mathrm k_{p}\dot{\mathbf p}_e^\mathrm{T} \mathbf p_e-2\mathrm{k_o}\dot \eta_e+\mathbf v^\mathrm{T} \mathbf F_{h}.
\end{split}
\end{equation} 
Utilising the property $\dot \eta_e=-\dfrac{1}{2}\vect{\epsilon}_e^\mathrm{T}\vect \omega_e$ where $ \vect \omega_e = \vect \omega_g - {\mathbf R}_e \vect \omega$ from \cite{koutras2021exponential} and substituting 
$\dot{\mathbf p}_e $ with $\dot{\mathbf p}_g-\dot{\mathbf p} $, \eqref{eq:dotV12} becomes: 

\begin{equation} \label{eq:dotV4} 
\dot V =-\mathbf v^\mathrm{T}\mathbf D_{d} \mathbf v+ \mathrm k_{p}\dot{\mathbf p}_g^\mathrm{T}\mathbf p_e+\mathrm k_{o} {\vect \omega}_g^\mathrm{T} \vect{\epsilon}_e+\mathbf v^\mathrm{T} \mathbf F_{h}.
\end{equation} 
Substituting $\dot{\mathbf p}_g$ and $ {\vect \omega}_g$ from \eqref{eq:dyn_god} yields:
\begin{equation} \label{eq:dotV42} 
\begin{split}
\dot V= &-\mathbf v^\mathrm{T}\mathbf D_{d} \mathbf v-\mathrm k_{r}\mathrm k_{p} f(a,\mathbf p_e) \mathbf p_e^\mathrm{T} \mathbf p_e\\&-\mathrm k_{r}\mathrm k_{o} f(a,\mathbf p_e) \vect{\epsilon}_e^\mathrm{T} \vect{\epsilon}_e+\mathbf v^\mathrm{T} \mathbf F_{h}.
\end{split}
\end{equation} 
Since  $f(a,\mathbf p_e)$ in not negative, and  $\mathrm k_{r}$ and $\mathrm k_{o}$ are positive, $\dot V$ can be written as:
\begin{equation} \label{eq:dotV55} 
\begin{split}
\dot V\leq &-\mathbf v^\mathrm{T}\mathbf D_{d} \mathbf v+\mathbf v^\mathrm{T} \mathbf F_{h}.
\end{split}
\end{equation} 
Therefore,  the inequality \eqref{eq:dotV55} implies that the system is strictly passive (refer to Definition 6.3 in \cite{Khalil_book}), thereby completing the proof of statement 1. 

2) Rewriting \eqref{eq:dotV55} by completing the squares, yields:

 	\begin{equation}\label{vdot_last_squares}
 	\begin{split}
 	\dot{V}=&-||\sqrt{\vct{D}_d} \mathbf v -\dfrac{1}{2}\sqrt{\vct{D}_d}^{-1} \vct F_{h} ||^2+\dfrac{1}{4}  \vct F_{h}^T \vct D_d^{-1} \ \vct F_{h} \\ &\leq \dfrac{\lambda_{max}(\vct D_d^{-1})}{4}  \vct F_{h}^\mathrm{T}  \ \vct F_{h}. 
 	\end{split}
 	\end{equation} 
  where $\lambda_{max}(.)$ is the largest eigenvalue of the matrix $"(.)"$. Notice that $\mathbf F_{h}$ represents the generalized force applied by the human
to guide the robot. The human forces have bounded energy. Hence integrating  \eqref{vdot_last_squares} we get:
\begin{equation} \label{eq:dotV57} 
 V\leq V(0)+ \dfrac{\lambda_{max}(\vct D_d^{-1})}{4} \int   \vct F_{h}^\mathrm{T} \vct F_{h} \ \partial t< \infty
\end{equation} which implies the boundedness of $V$.   Therefore, the boundedness of $V$  as given by  \eqref{eq:V5} indicates that  $\mathbf v$, $\mathbf p_e$ and $\Psi_e$ remain bounded under the exertion of human force. As a result,  deviations between the robot end-effector and the god-object pose are also bounded, completing the proof of statement 2.

\end{proof}

\subsection{Variable Damping} \label{subsec:var_dam}
The appropriate design of the damping term in the robot's admittance/impedance model is critical for ensuring a smooth and natural interaction experience, which is the focus of this subsection. Specifically, the damping matrix  $\mathbf{D}_d$  used in the robot end-effector admittance model \eqref{eq:adm_x2y} is constructed by combining a constant matrix   $\mathbf D_c \in \mathbb{R}^{6\times 6}$ and two variable matrices $\mathbf D_v=diag\{ d_{vp}
    \mathbf{I}_{3\times 3},   d_{vo} \mathbf{I}_{3\times 3}\}$, $\mathbf D_f=diag\{ d_{fp}
    \mathbf{I}_{3\times 3},   d_{fo} \mathbf{I}_{3\times 3}\}$, enhancing its ability to handle dynamic interaction scenarios. Therefore, the overall damping matrix is expressed as:
\begin{equation}\mathbf D_d = \mathbf D_c+\mathbf D_v + \mathbf D_f    \in \mathbb{R}^{6\times 6}.
\end{equation}
\subsubsection{Power-Based Variable Damping:}
Various formulations of the variable admittance matrix have been proposed in the literature. In this work, we employ a power-based variable damping matrix $\mathbf D_v$ proposed in \cite{sidiropoulos2021variable}.  The damping terms for translational and rotational motions are given by:
\begin{equation}
  d_{vp}=\mathrm a_p exp(-\mathrm b_p s_p), \  d_{vo}=\mathrm a_o exp(-\mathrm b_o s_o)  
\end{equation}
where $s_p= \mathrm{max}(0, \dot{\mathbf p}^\mathrm{T} \mathbf f_h )$ represents the positive power associated with the linear motion of the robot,  $s_o=\mathrm{max}(0, \vect \omega ^\mathrm{T} \vect\tau_h )$ is the positive power related to the rotational motion, and $\mathrm a_p, \mathrm a_o, \mathrm b_p , \mathrm b_o \in \mathbb{R}_>$ are scaling factors.

\subsubsection{Ergonomics-Based Variable Damping:}
In addition to power-based damping, the ergonomics factor is also considered in the damping design to prevent potential oscillations during the interaction. When $ \mathbf f(a,\mathbf p_e)$ deviates from one, the deviation  $\mathbf p_e$ between the god-object and admittance position may increase, raising the robot's resistance. This resistance can elevate the human's internal stiffness, potentially inducing oscillations.   To mitigate these effects, the variable damping terms $d_{fp}$ and $ d_{fo}$ are designed as:
\begin{equation}
 d_{fp}=\mathrm c_p ||\mathbf p_e||(1-f(a,\mathbf p_e)), \ 
  d_{fo}=\mathrm c_o ||\vect \epsilon_e||(1-f(a,\mathbf p_e))  
\end{equation}
where $\mathrm c_p, \mathrm c_o \in \mathbb{R}_>$ are positive gains. This formulation creates an adaptive damping mechanism that adjusts based on the deviation  magnitude and the value of the function $f(a,\mathbf p_e)$.

\subsection{Online Continuous Ergonomics Assessment}
The ergonomics factor $a$ is a critical component of the proposed postural virtual fixture. In order to calculate it, we first calculate the key joint angles necessary for computing it and then introduce a continuous approach for its calculation.

Let us assume we know the shoulder $\mathbf{S}_i$, elbow $\mathbf{E}_i$, wrist $\mathbf{W}_i$, neck $\mathbf{Ne}$,  middle of the thorax $\mathbf{Th}$, middle of the pelvis $\mathbf{Pl}$, and knee $\mathbf{Kn}$ keypoints, where $i=\mathrm{\{R,L\}}$ indicates the right or left side of the human body, respectively (Figure \ref{fig:human_angle}). This can be achieved using vision-based $3$D skeleton tracking algorithms, such as YOLO and MediaPipe, or using wearable motion capture sensors like the inertial measurement unit-based Xsens suit. Using these keypoints, the shoulder abduction/adduction angle $\theta^{i}_a$, shoulder flexion/extension angle $\theta^{i}_{f}$, shoulder internal/external rotation angle $\theta^{i}_{r}$, elbow flexion/extension angle $\theta^{i}_{e}$, and sagittal bending angle $\theta_{b}$ (see Figure \ref{fig:human_angle}) can be computed based on the geometric relationships between the joints. These calculations are provided in the \textit{Joint Motion Analysis} subsection of the Appendix.

 \begin{figure}[!t]
     \centering
       \includegraphics[width=0.42\textwidth]{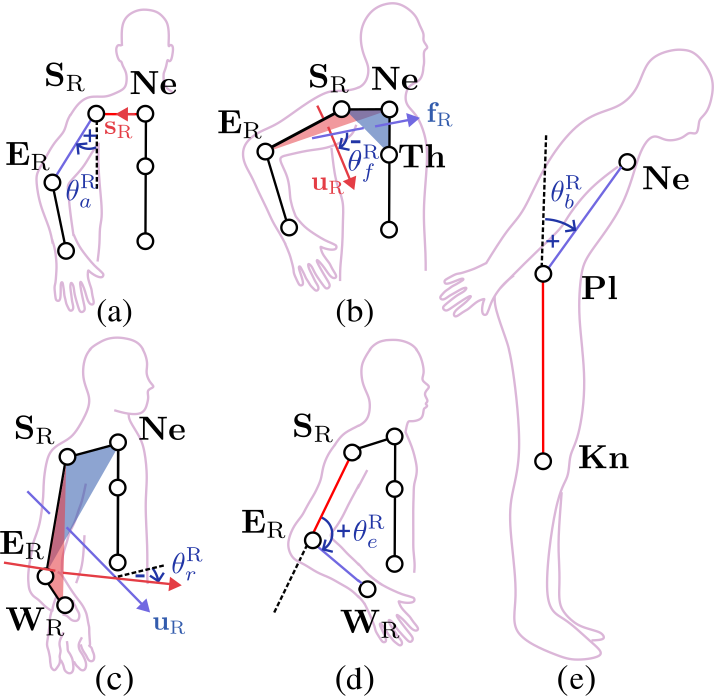}
        \caption{Analysis of human joint angles: (a) shoulder abduction/adduction, (b) shoulder flexion/extension, (c) shoulder internal/external rotation, (d) elbow flexion/extension, and (e) sagittal bending, across five movement scenarios. The figure illustrates key points identified by perception sensors, representing joint positions during each movement.}
         \label{fig:human_angle}
\end{figure}

The risk level associated with the angle of the upper arm, lower arm, and forward trunk bending joints, as indicated by the RULA worksheet \citep{mcatamney1993rula} is utilized to design the factor $a$. The RULA worksheet provides a discrete score based on the range in which each angle falls; then, using a score table, all the scores for each body part are gathered to compute a discrete total score that evaluates the ergonomics of the body posture. In our approach, we move away from this discrete method of posture assessment. Instead, we formulate a continuous function ranging from zero to one, with one indicating an ergonomic position and zero representing a high-risk, non-ergonomic posture. This continuous function allows for a smoother integration of the ergonomic assessment factor into the robot end-effector admittance model.
 Note that other bounded continuous ergonomic metrics, such as the joint displacement metric from \cite{lorenzini2022online} or the range of motion ergonomic cost from \cite{lagomarsino2023maximising}, could potentially be utilized in place of factor $a$ in the proposed virtual fixture dynamics. However, these metrics assume that the best ergonomic position of each joint is at the midpoint of its range of motion, which is not generally true for every joint. Therefore, our metrics, which are constructed based on the RULA assessment and specifically designed for online, posture-based ergonomics evaluation (facilitated by the SRB grounding the payload and simplifying the evaluation to posture quality), are considered more appropriate for assessing human posture ergonomics. Moreover, unlike our proposed ergonomic metric, these alternatives lack smoothness, which can lead to abrupt changes in control behavior.
Furthermore, fatigue-related metrics are unsuitable, as they may indicate ergonomic risk even when the user's current posture is ergonomic, due to the cumulative effect of previous non-ergonomic postures. This could hinder ergonomics recovery and task continuation.

 \begin{figure}[!t]
     \centering
       \subfloat[Ergonomics factor for shoulder  abduction/adduction angle.]{ \includegraphics[width=0.4\textwidth]{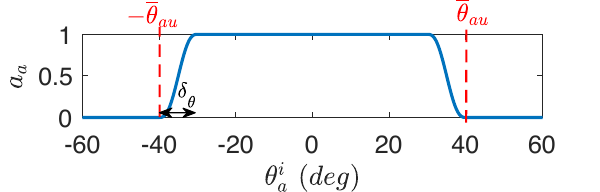}}
       \\
  \subfloat[Ergonomics factor for shoulder  flexion/extension angle.]{ \includegraphics[width=0.4\textwidth]{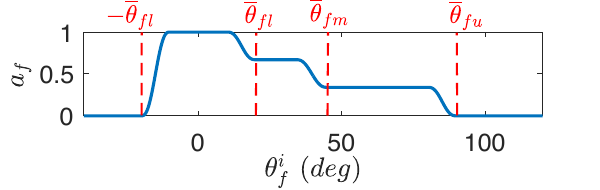}}
   \\
        \centering
       \subfloat[Ergonomics factor for shoulder   internal/external rotation angle.]{ \includegraphics[width=0.4\textwidth]{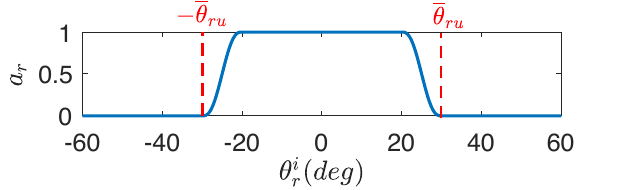}}
     \\
       \subfloat[Ergonomics factor for elbow flexion/extension angle.]{ \includegraphics[width=0.4\textwidth]{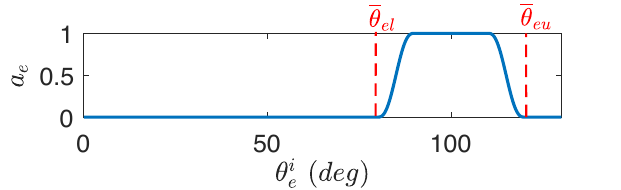}}
       \\
  \subfloat[Ergonomics factor for bending angle.]{ \includegraphics[width=0.4\textwidth]{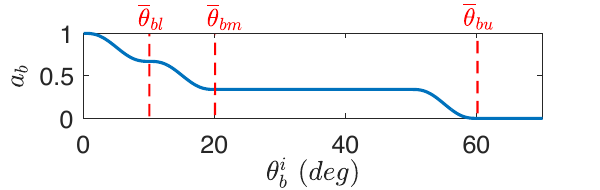}}
        \caption{Virtual representation of the ergonomics sub-factors.}
         \label{fig:task_des}
\end{figure}

For each angle calculated as described in the previous subsection, a function that calculates the posture ergonomics is designed based on the RULA worksheet. A virtual representation of these functions is presented in Figure  \ref{fig:task_des}. In particular,  for the shoulder abduction/adduction angle $\theta_a^i$, no specific joint limit is mentioned in the RULA worksheet; however, it is noted that less rotation is preferable. Thus, the following ergonomics factor is utilized for the angle $\theta_a^i$:
\begin{equation}
    \label{eq:a_a}
    a_a=h(|\theta_a^i|;\overline{\uptheta}_{au} -\updelta_\theta,\overline{\uptheta}_{au})
\end{equation}
where $\overline{\uptheta}_{au}$ denotes the upper ergonomic limit of the shoulder abduction/adduction angle, beyond which the posture is deemed non-ergonomic, resulting in $a_a = 0$. The ergonomics factor $a_a$ transitions smoothly from 1 to 0 within the range defined by $\updelta_\theta$. For this assessment, we set $\overline{\uptheta}_{au} = 30^\circ$ and $\updelta_\theta = 10^\circ$. 
For the shoulder flexion/extension $\theta_f^i$, a more ergonomic position is  when the angle is in a region around zero, and the risk level is increased as the angle is in an increased region; thus it is  designed as follows:
\begin{equation}
\label{eq:a_f} 
    a_f =
    \begin{cases} 
    \begin{aligned}
    &0.33h(\theta_f^i; \overline{\uptheta}_{fl} - \updelta_\theta, \overline{\uptheta}_{fl}) \\
    &+ 0.33h(\theta_f^i; \overline{\uptheta}_{fm} - \updelta_\theta, \overline{\uptheta}_{f2}) \\
    &+ 0.34h(\theta_f^i; \overline{\uptheta}_{fu} - \updelta_\theta, \overline{\uptheta}_{fu}),
    \end{aligned}
    & \text{if } \theta_f^i > 0 \\
    1 - h(\theta_f^i; -\overline{\uptheta}_{fl}, -\overline{\uptheta}_{fl} + \updelta_\theta),
    & \text{otherwise}
    \end{cases}
\end{equation}
where based on the RULA worksheet  $\overline{\uptheta}_{fl}=20^o$, $\overline{\uptheta}_{fm}=45^o$ and $\overline{\uptheta}_{fu}=90^o$.
For shoulder internal/external rotation $\theta_r^i$, 
no specific joint limit is mentioned in the RULA worksheet; however, it is noted that less rotation is preferable. Therefore, the ergonomics factor  is set as follows:
\begin{equation}
\label{eq:a_l} a_r=h(|\theta_r^i|;\overline{\uptheta}_{ru} -\updelta_\theta,\overline{\uptheta}_{ru})
\end{equation}
where the angle bound  $\overline{\uptheta}_{ru}$ is set to $\overline{\uptheta}_{ru}=30^o$ since the physical limit of absolute value of this shoulder angle is  $60^o$.

For elbow flexion/extension $\theta_e^i$, the factor is designed as: 
\begin{equation}
\label{eq:a_e}
    a_e=-h(\theta_e^i;\overline{\uptheta}_{el},\overline{\uptheta}_{el}+\updelta_\theta)+h(\theta_e^i;\overline{\uptheta}_{eu} -\updelta_\theta,\overline{\uptheta}_{eu})
\end{equation}
where based on the RULA worksheet  $\overline{\uptheta}_{el}=80^o$ and $\overline{\uptheta}_{eu}=120^o$.

For the bending angle $\theta_b$, the ergonomics factor RULA worksheet states that only a zero-degree angle is considered ergonomic.  To have a smooth transition, we add an extra range of motion around the zero-angle in addition to what is defined in the RULA worksheet. Consequently,  the ergonomics factor is defined as:
\begin{equation}
\label{eq:a_b}
\begin{split}
a_b=&0.33h(\theta_b;\overline{\uptheta}_{bl} -\updelta_\theta,\overline{\uptheta}_{b1})+0.33h(\theta_b;\overline{\uptheta}_{bm} \\&-\updelta_\theta,\overline{\uptheta}_{bm})+0.34h(\theta_b;\overline{\uptheta}_{bu} -\updelta_\theta,\overline{\uptheta}_{bu}).
\end{split}
\end{equation}
where $\overline{\uptheta}_{bl}=10^o$, $\overline{\uptheta}_{bm}=20^o$, and $\overline{\uptheta}_{bu}=60^o$.

To calculate the total ergonomics assessment score $a$, all sub-ergonomics factors are multiplied:
\begin{equation}\label{eq:a_t}
a = a_a a_f a_l a_e  a_b \in [0 \ 1].
\end{equation}
By doing this,  when at least one of them is at high risk, i.e., is zero,  the total score is zero, indicating a high-risk non-ergonomic posture.
\begin{remark} In this work, we consider the shoulder, elbow, and trunk bending angles. If the task requires the inclusion of additional angles, such as those of the wrist and legs, the ergonomics functions designed for the shoulder, elbow, and trunk can similarly be applied to these other angles. The total factor is then adjusted by multiplying it with the ergonomics factors corresponding to each of these angles. \end{remark}

\subsection{Supernumerary Robotic Body Control with Base Adaptation} \label{subsubsec:inv_k}
The final element of the control framework (Figure \ref{fig:block-diagram}) is an SRB with loco-manipulation capabilities, providing user support. This subsection addresses the closed-loop inverse kinematics with redundancy resolution for SRBs, ensuring the desired end-effector behavior \eqref{eq:adm_x2y} and coordination between the user and robot. The system integrates robotic platforms with $n \geq 6$ degrees of freedom and a mobile base, and an interface,  proposed in \cite{giammarino2024super}, incorporating end-effectors like the Pisa/IIT SoftHand or vacuum grippers and an F/T sensor measures human-exerted forces, $\mathbf{F}_h$, for the ergonomics-driven model \eqref{eq:adm_x2y}. The interface decouples human forces from environmental ones, enabling the handling of heavy loads with minimal effort using high-payload manipulators, such as the UR16e ($16$ kg capacity).

Solving the inverse kinematics problem is crucial for translating the desired end-effector behavior, as defined by \eqref{eq:adm_x2y}, into precise SRB joint velocities $\dot{\mathbf{q}}_p \in \mathbb{R}^{3+n}$. These velocities are obtained using the Damped Least-Squares (DLS) method, detailed in the \textit{DLS for Whole-Body Inverse Kinematics} subsection of the Appendix. The velocity vector $\dot{\mathbf{q}}_p$ represents the complete joint velocity configuration of the robot, defined as  $\dot{\mathbf{q}}_p = [\dot{\mathbf{q}}_b^\mathrm{T} \ \dot{\mathbf{q}}_a^\mathrm{T}]^\mathrm{T} \in \mathbb{R}^{3+n}$, where $\dot{\mathbf{q}}_b = [v_x \ \ v_y \ \ \omega_z]^\mathrm{T} \in \mathbb{R}^3$ corresponds to the floating-base velocities, and $\dot{\mathbf{q}}_a \in \mathbb{R}^n$ represents the manipulator joint velocities.

The SRB's redundancy can be leveraged to improve coordination, especially when the operator’s motion is constrained by the floating base. To address this, we use the Jacobian's null space $\mathcal{N} \in \mathbb{R}^{\{3+n\}\times \{3+n\}}  $, provided in the  \textit{DLS for Whole-Body Inverse Kinematics} subsection of the Appendix, to adjust the base’s movement as the operator approaches a threshold. The total commanded joint velocity is:
\begin{equation} \dot{\mathbf{q}}_c = \dot{\mathbf{q}}_p + \mathcal N \mathbf{v}_r \in \mathbb{R}^{3+n} \end{equation}
where $\mathbf{v}_r = [\mathbf{v}_{xy}^\mathrm{T} \ \mathbf{0}_{(1+n)}]^\mathrm{T} \in \mathbb{R}^{3+n}$ with  $\mathbf{v}_{xy} \in \mathbb{R}^2$ be repulsive velocity component that introduced into the null space of the Jacobian when the user is near the base. Then, the velocity $\dot{\mathbf{q}}_c$ is integrated based on the robot’s motion input (joint velocities or positions).

\begin{figure}[!t]
     \centering
\includegraphics[width=0.75\linewidth]{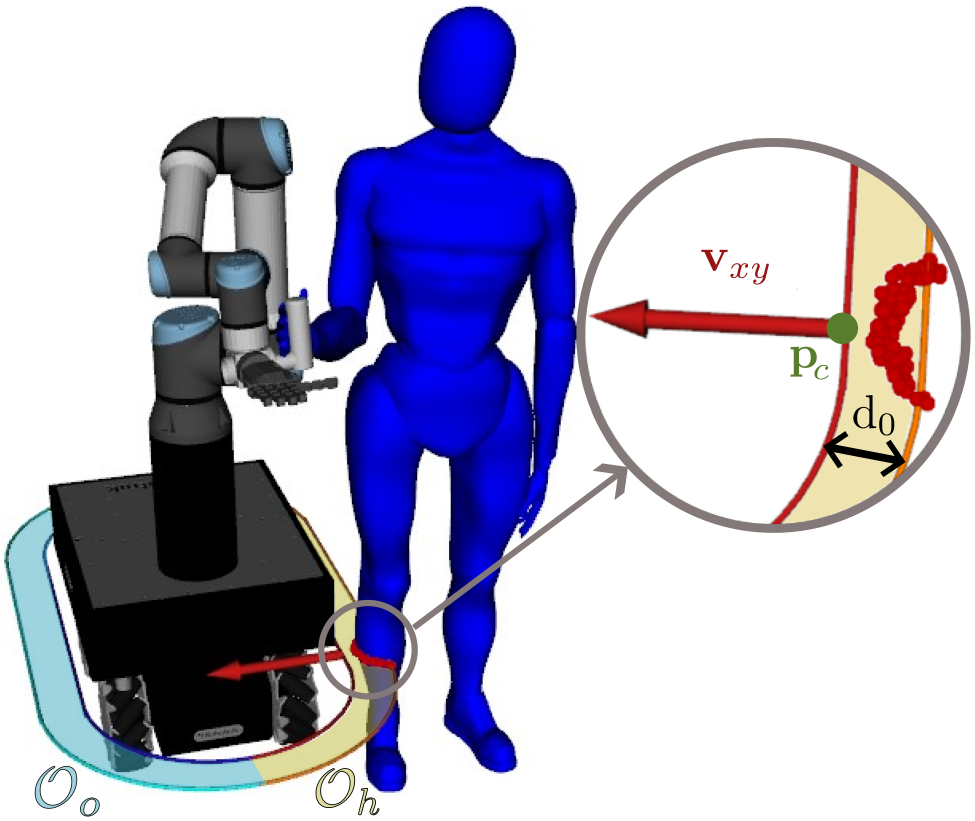}
        \caption{Visualization of the repulsive vector when the user is in close proximity to the robotic platform.
        }
         \label{fig:nullspace}
\end{figure}
To calculate the repulsive velocity $\mathbf{v}_{xy}$, we drew inspiration from \cite{sirintuna2024enhancing}, where the goal was for a human-mobile manipulator team to avoid obstacles in the environment during a co-transportation task. Without loss of generality, we consider as a floating base the robotic mobile platform shown in Figure \ref{fig:nullspace}. In contrast to that previous work, which filtered the velocity generated from an admittance model based on whether the robot-human team is expected to collide with an obstacle—thereby constraining the robot's motion to ensure safety—our current approach involves designing a vector that repels the robot platform when a human approaches its base in a distance less than $\mathrm {d}_{0}  \in
    \mathbb{R}$.   To implement this, we first enclose the robot in a capsule with radius $\mathrm{r}_c \in \mathbb{R}^3$ and length $\mathrm{L} \in \mathbb{R}$, corresponding to the robot's dimensions.  Following this, we fill the empty space created by the LiDAR-detected point cloud by drawing circles with a radius 
$\mathrm{r}_s \in \mathbb{R}$ around these points.  By doing so, we can identify the region of interest, which is the area close to the capsule, defined  as: \begin{equation}
    \mathcal{O}_c=\{\mathbf x\in \mathbb R^2: \mathrm{r}_c+\mathrm{r}_s \leq ||\mathbf x||<\mathrm  {d}_{0}  +\mathrm{r}_s+\mathrm{r}_c\}.
\end{equation} 
\begin{remark}
       For simplicity, the robotic platform is enclosed within a capsule-shaped boundary. However, if a more precise fit is needed, other shapes, such as 2D superellipses, could be employed. These shapes have been shown to be effective for obstacle avoidance in autonomous robotic tasks involving fixed robotic manipulators \citep{stavridis2017dynamical}.
\end{remark}
 
 To form the repulsive velocity, we consider the region that corresponds to half of the capsule where a human is expected to interact with the robot end-effector, denoted as 
$\mathcal{O}_h\subset\mathcal{O}_c$. When points are detected in the opposite half $\mathcal{O}_o=\mathcal{O}_c-\mathcal{O}_h$, the null space control is deactivated to ensure that the environment remains safe from potential collisions and hazards. To ensure smooth behavior of the repulsive velocity when points are detected only in the half capsule closest to the human $\mathcal{O}_h$,
instead of considering just the closest point, we compute a
weighted average of all nearby points:  
\begin{equation}
    \bar{\mathbf p}=\dfrac{\sum_{ \mathbf p_i\in  \mathcal{ O}_h } w_i \mathbf p_i^c}{\sum_{ \mathbf p_i\in  \mathcal{ O}_h } w_i } \in \mathbb{R}^2, \ w_i = h(d_{xi}; 0,{\mathrm {d}_{0} })  \in \mathbb{R}
\end{equation}
where  $\mathbf p_i^c \in \mathbb{R}^2$  is the closest point on the perimeter
of the capsule to $\mathbf p_i \in  \mathcal{ O}_h $  that can be calculated analytically  (see Appendix) and $d_{xi} = ||\mathbf{p}_i - \mathbf{p}_i^*|| - (\mathrm{r}_c + \mathrm{r}_s)\in [0  \ \mathrm{d}_{0}  )$ with $\mathbf{p}_i^*$ is the nearest point on the capsule's centerline to $\mathbf{p}_i$, also calculable analytically   (see subsection \textit{Closest Point on a 2D Capsule} in the Appendix). This weighting scheme ensures that points closer to the robot have a greater influence on $\bar{\mathbf p}$. Then, this point is projected onto the capsule perimeter as:
\[
\mathbf{p}_c = \bar{\mathbf{p}}^* + \frac{\bar{\mathbf{p}} - \bar{\mathbf{p}}^*}{\|\bar{\mathbf{p}} - \bar{\mathbf{p}}^*\|} \mathrm{r}_c    \in \mathbb{R}^2 , 
\]
where $\bar{\mathbf{p}}^* \in \mathbb{R}^2$ is  the nearest point on the capsule centerline to $\bar{\mathbf{p}}$.
Finally, The repulsive force vector $\mathbf{v}_{xy}$ is given by:
\begin{equation}
    \mathbf v_{xy}=k_v \mathbf R^b_w\dfrac{ \mathbf p_c^* - \mathbf p_c}{||\mathbf p_c^*  -\mathbf p_c||} \in \mathbb{R}^2
\end{equation}
where ${\mathbf{p}}^*_c \in \mathbb{R}^2$ is the nearest point on the capsule line to ${\mathbf{p}}_c$, $\mathbf R^b_w$  is the 2D matrix mapping the x-y axis vectors expressed in the world frame to the mobile robot's base frame, and $k_v =\mathrm a_k h\left( \min\limits_{\mathbf{p}_i \in \mathcal{O}_h} \{ d_{xi} \}; 0, \mathrm{d}_{0}  \right)
$ with $\mathrm a_k \in \mathbb{R_>} $  being a gain factor. Notice that $k_v$ is a gain that increases as the distance between the human and the platform decreases.

\begin{remark}
Here, we consider a predefined area $\mathcal{ O}_h$  with respect to the mobile platform where the user is expected to interact with the robot. For more flexible solution leg detection methods utilizing 2D LiDARs \citep{leigh2015person}, could also be utilized in order  to adapt the area $\mathcal{ O}_h$.
\end{remark}

\section{EXPERIMENTS}
To validate the functionality and effectiveness of the proposed control framework, with respect to ergonomic performance and user interaction, we conducted two multi-subject experiments. 
The first, referred to as the \textit{Prolonged Manipulation Task} (Figure \ref{fig:s1}), aimed to quantify the ergonomics improvement enabled by the proposed postural virtual fixtures and to evaluate the overall user experience in comparison to a baseline admittance model. 
The second experiment, termed the \textit{Long Distance Loco-Manipulation Task} (Figure \ref{fig:exp2}), involved lifting and carrying a load over an extended distance. This task was designed to validate the ergonomics-driven robot whole-body adaptation during co-manipulation and to determine whether the conditional activation of the base-prioritization functionality promote ergonomic behavior during extended loco-manipulation, even in the absence of kinesthetic feedback.
The video of the presented experiments can be found in the multimedia attachment or at \href{https://youtu.be/fZYz6XOJ1Po}{https://youtu.be/fZYz6XOJ1Po}.

\subsection{Experimental Setup}

Both experiments used the Kairos mobile manipulator, which consists of a Robotnik SUMMIT-XL STEEL mobile platform paired with a 6 degrees of freedom (DoFs) Universal Robot UR16e manipulator. 
For grasping tasks, the Pisa/IIT SoftHand was attached as the end-effector. 
Additionally, a 3D-printed handle integrating an F/T sensor was positioned near the SoftHand to measure human forces. Mounted on top of the handle was an M5Stack interface, which allowed participants to control various control modalities including the switching between \textit{arm prioritization} mode (setting $\mathbf{W_q}= diag\{  10^4\mathbf I_{3\times3},  \mathbf I_{6\times 6}\} $, with the operator ${diag}\{ ...\}$  used to represent block diagonal matrices) and \textit{base prioritization} mode (setting $\mathbf{W_q}=diag\{ \mathbf I_{3\times3},  \mathbf I_{6\times 6}\} $) by touch, thus prioritizing either the arm joints or the DoFs of the mobile base, respectively. 
The mode switch was manually operated by the user for intuitive control based on immediate needs. While future work could explore autonomous switching via human action recognition (e.g., detecting motion size), this was not the focus of the proposed control framework. Additionally, the touch user interface enabled users to control the open or closed status of the SoftHand for grasping objects. The robotic base is also equipped with two 2D SICK Microscan LiDAR sensors to detect the proximity of the human or obstacles on the x-y plane.

To demonstrate that the proposed methodology is independent of the skeleton tracking approach, a
preliminary experiment was conducted using a single camera, while the remainder of the experimental
campaign relied on an IMU-based suit. However, a direct comparison between these motion capture
methods is beyond the scope of this work. For the single-camera setup,  an RGB-D (Intel RealSense) camera monitored operator motion at approximately 20 Hz. In that case,  we exploited YOLO algorithm to track the human skeleton and compute joint angles and the ergonomics factor as described in the previous section. However, YOLO did not provide positions for the neck ($\mathbf {Ne}$), the pelvis ($\mathbf {Pl}$), and the middle of the thorax ($\mathbf {Th}$). We calculated $\mathbf{Ne}$ as the midpoint between the shoulder points $\mathbf{S_R}$ and $\mathbf{S_L}$, $\mathbf{Pl}$ as the midpoint of the hips, and $\mathbf{Th}$ as the point two-third of the way between $\mathbf{Ne}$ and $\mathbf{Pl}$. For tasks requiring greater spatial coverage, the limited field of view and susceptibility to occlusions of a single camera necessitated more stable tracking. To address this, we then employed a wearable MVN Biomech suit (Xsens Tech.BV), which relies on inertial measurement unit (IMU) sensors, to track the skeleton keypoints. This system ensured reliable tracking even during occlusions or large displacements that could move participants out of the camera’s view, hence, contributed to a more consistent evaluation of the proposed algorithm in our multi-subject studies.

\begin{figure}[!t]
    \centering     \includegraphics[width=0.35\textwidth]{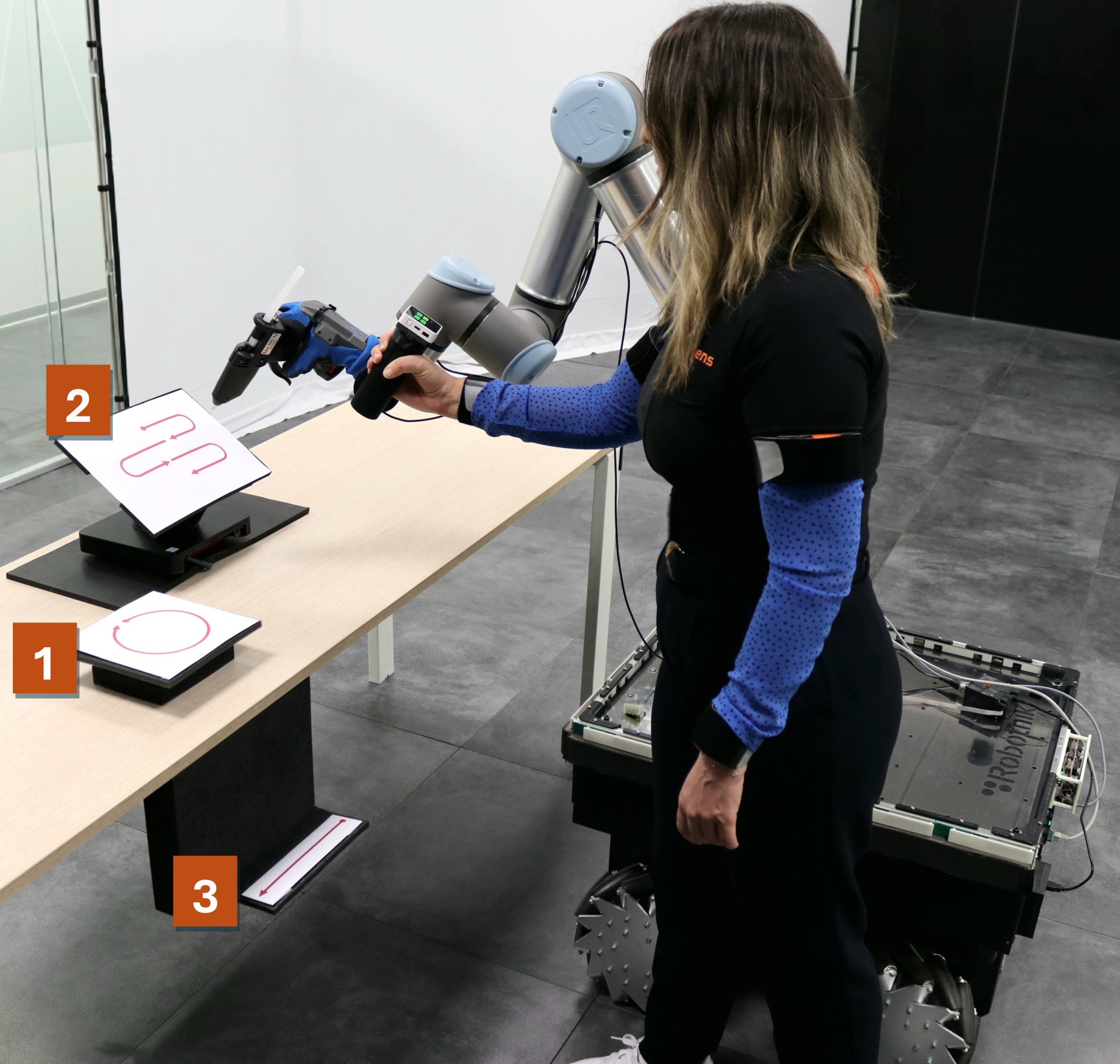}
    \caption{Setup of the \textit{Prolonged Manipulation Task}, including the human operator, SRB, glueing tool, Xsens devices, and target tracking paths with numbers indicating the order.}
    \label{fig:s1}
\end{figure}

\begin{figure}[!t]
    \centering 
\includegraphics[width=0.85\linewidth]{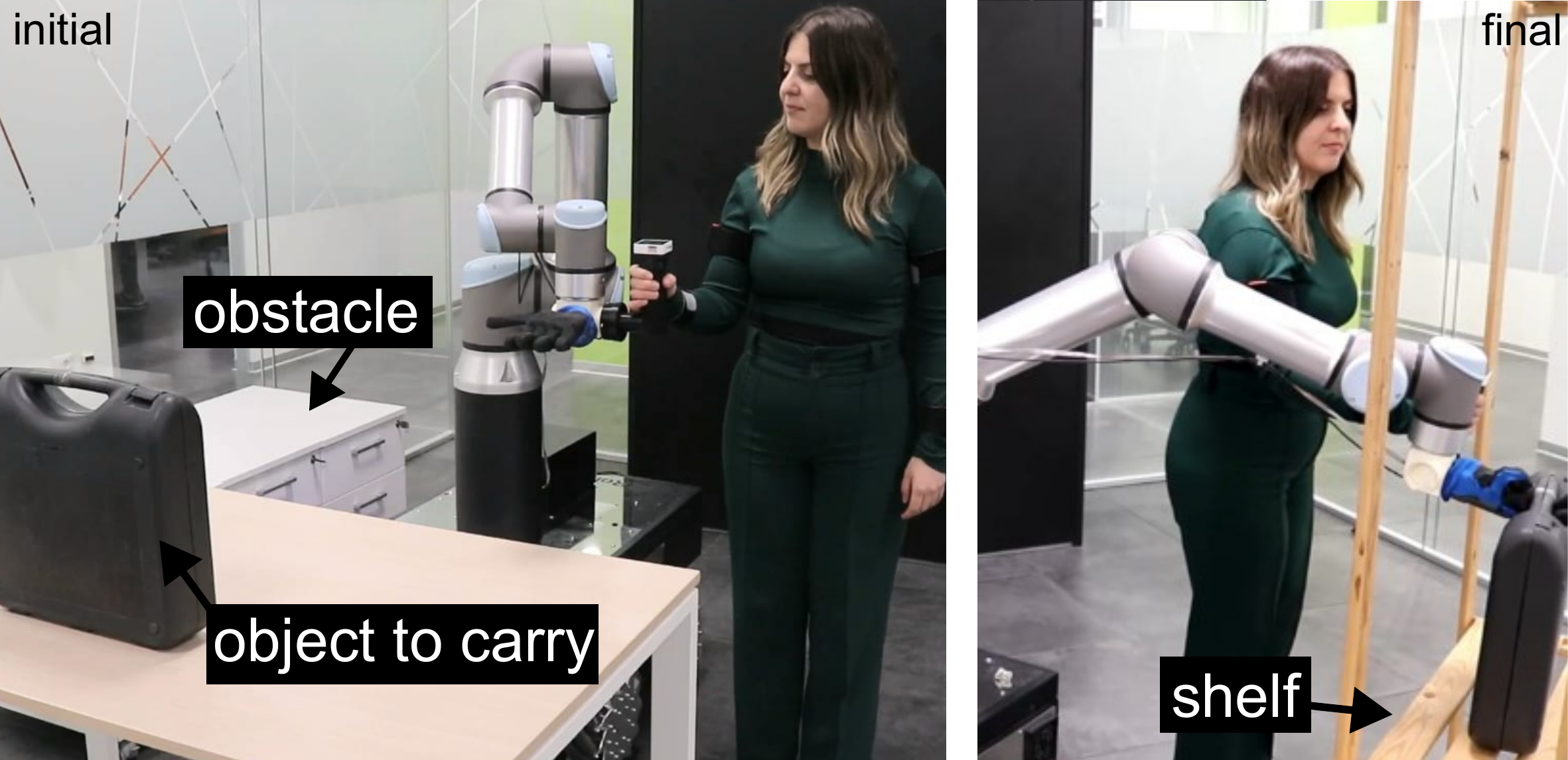}
\hspace{0.005\textwidth}
    \caption{Setup of the \textit{Long Distance Loco-Manipulation Task}, involving lifting a load, carrying it, and repositioning it onto a shelf. Note that the chests of drawers serve as obstacles in the environment.}
    \label{fig:exp2}
\end{figure}

\subsection{Experimental Protocol}

The \textit{Prolonged Manipulation Task} involved completing a series of target-tracking paths arranged in various configurations on a table, five consecutive times. This task was designed to showcase the potential of the ergonomics-driven pHRI module in tasks that require prolonged and precise arm movements. Such tasks, commonly performed with low- or medium-weight tools (e.g., for gluing or polishing), often lead to non-ergonomic postures and pose long-term health risks. Participants were instructed to complete three target paths in the sequence shown in Figure \ref{fig:s1}, with the SRB grasping a hotmelt glue gun. To do so, the users exerted forces at the SRB handle to guide the tool through the target paths. 
For this task,  the \textit{arm prioritization mode} was utilized with the parameter values of the proposed control scheme presented in Table \ref{table:parameters}. 
The parameters of the proposed control framework were selected to balance performance and user intuitiveness, ensuring responsiveness while maintaining ergonomic comfort. Considering static motion, the stiffness matrix has been defined in order to establish a proportional relationship: a 2 cm deviation from the last ergonomic posture in translational space results in 12 N force feedback, while a 4 $^o$ deviation in rotational space produces 1.4 Nm torque feedback. The damping values were experimentally determined to allow for smooth responses during rapid movements while minimizing oscillations when the ergonomic factor is zero. Similarly, the inertia values of the admittance model were chosen experimentally to provide a compliant, responsive, and intuitive user experience. This selection process was conducted before the main experiments using a subject outside the user study.  For the god object, the motion dynamics gain was tuned to achieve fast convergence at an ergonomic factor of 1, preventing motion constraints when the human posture is ergonomic.

In the \textit{Long Distance Loco-Manipulation Task}, the operator collaborated with the SRB to lift and transport a load across a larger workspace, as illustrated in Figure \ref{fig:exp2}. They were instructed to approach the box, grasp it, carry it to the designated drop location, and place it on a shelf, while obstacles were also present in the surrounding environment. 
To facilitate both precise manipulation and long-distance movement, the touch user interface was used to switch between \textit{arm prioritization mode} and \textit{base prioritization mode}. Transition from \textit{arm prioritization mode} to \textit{base prioritization mode} was permitted only when the user's ergonomics factor exceeded a predefined threshold ($a_{th} = 0.5$). During the \textit{base prioritization mode}, posture-based kinesthetic feedback was deactivated by setting $\mathbf{K}_d = 0$, as providing such feedback during locomotion was found to be potentially disruptive. Instead of offering real-time postural guidance while walking, the system ensured that switching occurred only when the user was already in an ergonomic posture, thereby promoting the adoption and maintenance of ergonomic postures throughout the task. The underlying hypothesis was that promoting ergonomics during manipulation could also positively influence posture during the \textit{base prioritization mode}, exploiting the natural tendency of humans to prioritize walking without significantly altering arm configurations during long transportation tasks. Since, according to the literature \citep{lamy2009achieving, lecours2012variable}, the minimum allowed target inertia value depends on the actual inertia of the physical plant in order for the overall system to remain passive, and because prioritizing the robotic base joints increases the actual inertia, an increase in the corresponding admittance gains is required. Therefore, we set the admittance matrices $\mathbf M_d$ and  $\mathbf{D}_c$ as given in Table \ref{table:parameters2} with the remaining parameters set to the same values as in the \textit{arm prioritization mode}.

\begin{table}[!t]
\small
    \caption{Value of parameters used in the experimental setup.}
    \label{table:parameters}
    \def \arraystretch{1.1}
        
    \begin{tabular}{p{0.5cm} p{3cm}}
    \hline
        \multicolumn{2}{l}{\hspace{-2mm}\textbf{Robot and God-object Motion}}\\
    \hline
        \hspace{-2mm}$\mathbf{M}_{d}$ & $diag\{5 \, \mathbf I_{3\times 3}, 0.25 \, \mathbf I_{3\times 3}\}$ \\ 
        \hspace{-2mm}$\mathbf{K}_{d}$        & $diag\{\mathrm 600 \mathbf I_{3\times 3}, 40 \, \mathbf I_{3\times 3}\}$ \\
        \hspace{-2mm}$\mathrm {k}_{r}$       & 200\\
        \hspace{-2mm}$\mathrm {m}$           & 0.5 \\ 
        \hspace{-2mm}$\mathrm {k}_{n}$       & $10^3$  \\  
        \hspace{-2mm}$\mathrm{\delta_n}$     & 5 deg \\
        \hspace{-2mm}$\overline{\upphi}_n$ & 55 deg \\
    \hline
    \end{tabular}
    \hfill
    \begin{tabular}{p{0.5cm} p{2.6cm}}
    \hline
        \multicolumn{2}{l}{\hspace{-2mm}\textbf{Damping term $\mathbf D_d$}}\\
    \hline
        \hspace{-2mm}$\mathbf{D}_{c}$ & $diag\{20 \, \mathbf I_{3\times 3}, \mathbf I_{3\times 3}\}$ \\
        \hspace{-2mm}$\mathrm a_p$  & 20 \\
        \hspace{-2mm}$\mathrm a_o $ & 4 \\
        \hspace{-2mm}$\mathrm b_p$  & 1\\
        \hspace{-2mm}$\mathrm b_o$  & 5\\
        \hspace{-2mm}$\mathrm c_p$  & 2500 \\
        \hspace{-2mm}$\mathrm c_o$  & 20 \\
    \hline
    \end{tabular} 
    \vspace{0.02cm}
\centering    
    \begin{tabular}{p{0.8cm} p{3.7cm}}
    \hline
        \multicolumn{2}{l}{\hspace{-2mm}\textbf{SRB Control Framework}}\\
    \hline
        \hspace{-2mm}$\mathrm{n}$       & 6 \\
        \hspace{-2mm}$\mathbf {K}_x$    & $diag\{0.01 \, \mathbf I_{1\times3}, 0.004 \ \mathbf I_{1\times 3}\}$  \\
        \hspace{-2mm}$\mathbf {W}_x$    & $10^3 \, \mathbf{I}_{6 \times 6}$   \\
        \hspace{-2mm}$\mathrm r_c$      & 37.5 cm \\   
        \hspace{-2mm}$\mathrm r_s$      & 2 cm\\ 
        \hspace{-2mm}$\mathrm {d}_{0} $ & 10 cm   \\
        \hspace{-2mm}$\mathrm L$        & 31 cm \\
        \hspace{-2mm}$\mathrm a_k$      & 0.11 \\
    \hline
    \end{tabular} 
\end{table}

\begin{table}[!t]
\small
    \caption{Parameter values during mobile \textit{base prioritization}.}
     \label{table:parameters2}
    \centering
    \begin{tabular}{p{0.8cm} p{3.7cm}} 
    \hline
    \multicolumn{2}{l}{\hspace{-2mm}\textbf{Adjusted parameter values}}\\
     \hline
     \hspace{-2mm}$ \mathbf{M}_{d}$ & $diag\{15\, \mathrm I_{3\times 3}, 0.5 \, \mathrm I_{3\times 3}\}$ \\
    \hspace{-2mm}$ \mathbf{D}_{c}$  & $diag\{40\mathrm \, \mathrm I_{3\times 3}, \mathrm I_{3\times 3}\}$ \\
    \hspace{-2mm}$ \mathbf{K}_{d}$  & $\mathbf 0_{6\times 6}$ \\
    \hspace{-2mm}$ \mathrm{c_p}, \ \mathrm{c_o}$ & $0$ \\
     \hline
    \end{tabular}
\end{table}

Fourteen healthy participants ($9$ males and $5$ females, aged $28.7 \pm 2.9$ years), all right-handed, were recruited for the user study.
The experimental campaign was carried out at the Human-Robot Interfaces and Interaction (HRII) Laboratory of IIT, in accordance with the revised Declaration of Helsinki. The protocol was approved by the ethics committee Azienda Sanitaria Locale (ASL) Genovese N.3 (Protocol IIT\_HRII\_ERGOLEAN 156/2020).

Both tasks were repeated twice by each participant under two different conditions: Postural Virtual Fixtures (PVFs) and Baseline (B). In the PVFs condition, kinesthetic feedback was provided to promote postures that avoided non-ergonomic configurations in \textit{arm prioritization mode}, and switching to \textit{base prioritization mode} was permitted only when the participant's posture met the ergonomic threshold ($a \geq a_{th}$). In the B condition, the system operated without integrating human posture in the control loop. To achieve this, the ergonomics factor $a$ was fixed at $1$ throughout the experiments. This caused the god-object to follow the desired end-effector motion, thereby not providing posture-based kinesthetic feedback. Mode switching from \textit{arm prioritization} to \textit{base prioritization} was always possible for the user, and there was no check on the user's posture.  
Before the main trials, participants underwent a familiarization phase with the admittance interfaces without being informed of their specific functionalities, allowing them to interpret the system's feedback on their own. To minimize learning effects and cumulative workload that could affect the statistical results, the order of the two conditions was randomized, and a break was provided between sessions.

\subsection{Measurements and Derived Metrics}

This subsection outlines the measurements and the metrics used to assess the proposed control framework. Statistical analysis was performed to compare our results with those obtained from the baseline framework (i.e., PVFs condition vs B condition) and identify any significant differences between the two. Initially, the gathered data were tested for normality using the Anderson-Darling test. For normally distributed data,  paired t-tests were employed for pairwise comparisons, and repeated measures ANOVA was used for metrics involving multiple repetitions, such as the five iterations per participant in the \textit{Prolonged Manipulation Task}. For data that did not meet the normality assumption, non-parametric alternatives were applied, including the Wilcoxon signed-rank test for pairwise comparisons and the Friedman test for repeated measures, to identify significant differences.

\subsubsection{Ergonomics Metrics:}
To assess overall ergonomic performance during the tasks, the mean ergonomics factor $\bar{a}$ was computed as the average value of the factor $a$ in \eqref{eq:a_t} over the entire task duration:
\begin{equation}
    \bar{a} = \frac{1}{T_k} \sum_{k=1}^{T_k} a[k],
\end{equation}
where $k$ represents the time samples and $T_k$ is the total number of samples during the task.

The percentage of non-ergonomic time $\zeta_{ne}$ was calculated to assess the effectiveness of the proposed method in promoting the avoidance of prolonged non-ergonomic postures, which are associated with increased health risks. This metric is expressed as:
\begin{equation}
    \zeta_{ne} = \left(\frac{1}{T_k} \sum_{k=1}^{T_k} \mathbb{I}(a[k] = 0) \right) \times 100,
\end{equation}
to quantify the proportion of the task duration spent in non-ergonomic postures ($a = 0$). $\mathbb{I}(\cdot)$ is the indicator function, which equals $1$ when the condition in parenthesis is true, i.e., in this case $a[k] = 0$, and $0$ otherwise.

To capture how frequently the user transitioned into non-ergonomic postures, the non-ergonomic boundary touch count $\beta$ was calculated. This metric counts the instances when the ergonomics factor dropped from a valid ergonomic state to a non-ergonomic state, defined as:
\begin{equation}
    \beta = \sum_{k=2}^{T_k} \mathbb{I}(0 < a[k-1] \leq 1 \, \land \, a[k] = 0).
\end{equation}
Additionally, the percentage of time the robotic base was within a critical proximity threshold was computed to estimate potential interference with human manipulation. This metric is expressed as:
\begin{equation}
    \zeta_{d} = \left(\frac{1}{T_k} \sum_{k=1}^{T_k} \mathbb{I}(\mathrm{d}[k] < \mathrm{d}_{0}) \right) \times 100,
\end{equation}
where $\mathrm{d}[k]$ is the distance on the horizontal plane between the user's leg and the robotic base at each time step $k$, and $\mathrm{d}_{0}$ is the defined proximity threshold.

\subsubsection{Subjective Rating Scales:}

At the end of the experiments, participants were asked to complete a custom questionnaire (Table \ref{table:custom_quest}) designed to evaluate their experience with the SRB assistance across the two testing conditions. To mitigate bias, we counterbalanced positive and negative items, assessing key factors such as the promotion of postural ergonomics, learning effects, control freedom, and the effectiveness of null-space adaptation. Selected items from the System Usability Scale (SUS) \citep{lewis2018system} were also included in the questionnaire to examine whether the PVFs impacted the system's usability. All feedback was collected using a five-point Likert scale. Additionally, the NASA Task Load Index (NASA-TLX) questionnaire \citep{hart2006nasa} was used to evaluate participants' perceived workload.

\begin{table}[!h]
\small
    \caption{}
    \label{table:custom_quest}
    \def \arraystretch{1.1}
    \centering
    \begin{tabular}{p{0.4cm}p{7cm}}
        \hline
        \multicolumn{2}{l}{\textbf{Custom Questionnaire }}\\
        \hline
        \textit{Q1}: & The interface effectively helped me adopt and maintain an ergonomic posture during the task.\\
        \textit{Q2}: & It was challenging to recognize when I was close to a non-ergonomic posture.\\
        \textit{Q3}: & It was easy to correct my posture to complete the task ergonomically.\\ 
        \textit{Q4}: & After a few repetitions, I had learned to be more ergonomic during the task.\\ 
        \textit{Q5}: & I felt in control of the system while performing the task.\\  
        \textit{Q6}: & The robotic mobile base interfered with my movement and task completion.\\ 
        \textit{Q7}: & I was concerned that the robotic mobile base could collide with obstacles, e.g., walls.\\ 
        \hline
    \end{tabular}
    \begin{tabular}{p{0.4cm}p{7cm}}
        \hline
        \multicolumn{2}{l}{\textbf{System Usability Scale (SUS)}}\\
        \hline
        \textit{SUS1}: & I found the various functions of the system were well integrated.\\ 
        \textit{SUS2}: & I thought the interface was easy to use.\\ 
        \textit{SUS3}: & I felt very confident using the system.\\ 
        \hline
    \end{tabular}
\end{table}

\section{EXPERIMENTAL RESULTS}

\subsection{Experiment 1: Partial Prolonged Manipulation Task -  RGBD Camera Skeleton Tracking}

To showcase the functionality of the proposed control scheme and highlight its flexibility to different motion
capture systems, we first used vision-based skeleton tracking to calculate the ergonomics factor.  A right-handed  expert user performed one subsection of the gluing task, following the third target path, which was positioned at a low height.  The results are illustrated in Figures \ref{fig:pathcam} and \ref{fig:cam_res}.
\begin{figure}[!t]
     \centering \includegraphics[width=0.43\textwidth]{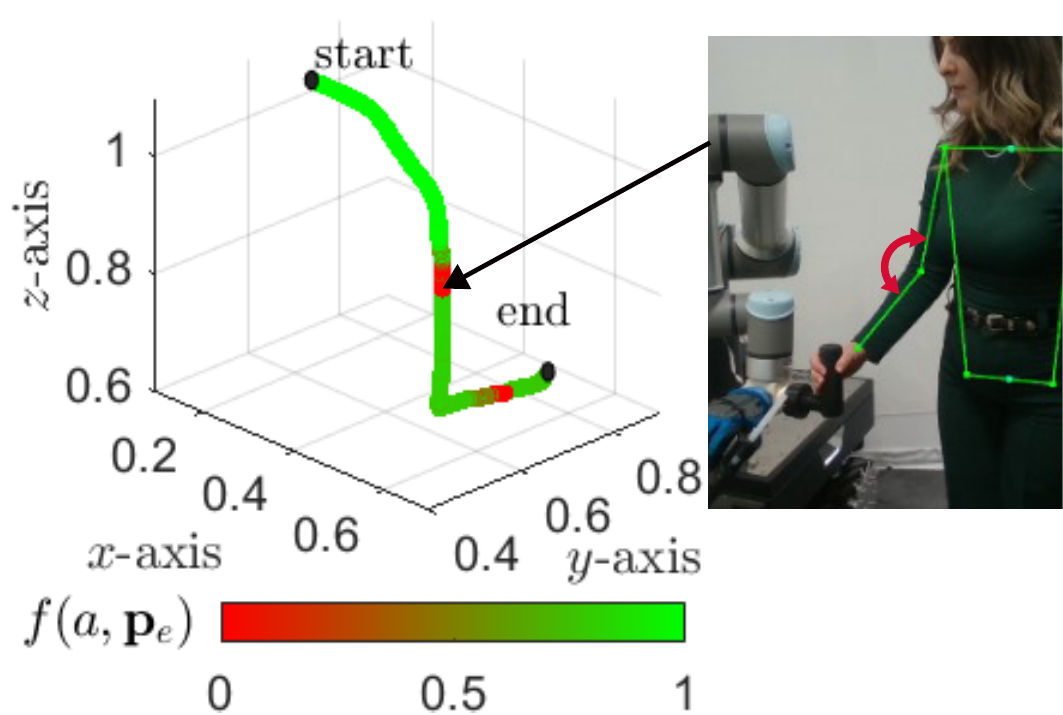}
\caption{Visualization of the robot end-effector trace with color gradients representing the values of $f(a,\mathbf p_e)$, alongside an RGB-D camera image processed with YOLO for skeleton detection (Experiment 1), marking a time-stamped instance when the user's posture was non-ergonomic due to elbow extension.}
\label{fig:pathcam}
\end{figure}
\begin{figure}[!t]
    \centering    
    \subfloat[From top to bottom, the plots display the human joint angles ($\theta_r, \theta_e, \theta_b$) with their corresponding ergonomic ranges, the ergonomics factor $a$, the norm of the deviations of the robot's position from the god-object position ($||\mathbf p_e||$), and the norm of the human-exerted force ($||\mathbf f_h||$).]{ 
\includegraphics[width=0.45\textwidth]{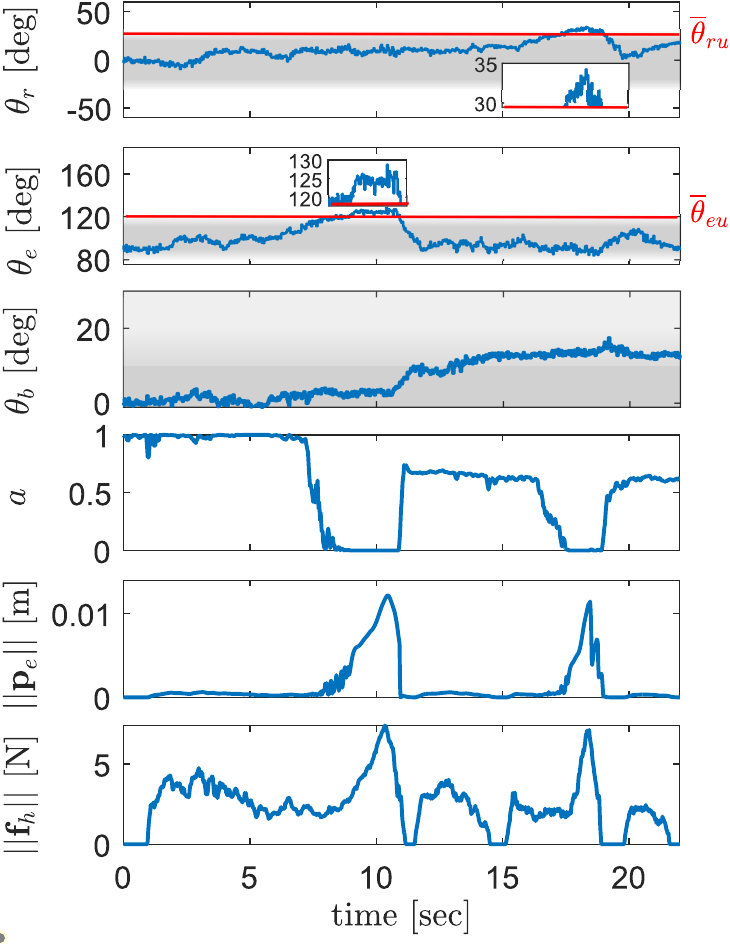}}
    \\ 
    \subfloat[Visualization of the damping components.]{ \includegraphics[width=0.36\textwidth]{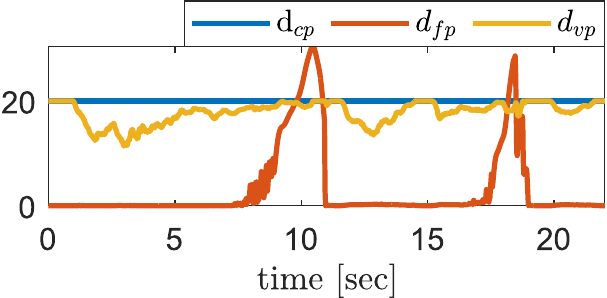}}
    \caption{Experimental results utilizing an RGB-D camera and YOLO detector for skeleton tracking (Experiment 1).}
\label{fig:cam_res}
\end{figure} 
In Figure \ref{fig:pathcam}, the trace of the robot end-effector and a timestamp of a human non-ergonomic posture are shown. The color of the end-effector's trace varies based on the output of the function $f(a,\mathbf p_e)$, providing immediate visual feedback about the function's value at different points. The image represents a posture where the user is at risk due to the elbow angle exceeding the upper threshold.

Figure \ref{fig:cam_res} (a)  shows a series of plots depicting, from top to bottom,  the most relevant human joint angles for this task, the total ergonomics factor, the norm of the position deviation between the robot end-effector and the god-object position, and the norm of human-exerted force throughout the task. The angles $\theta_a$ and $\theta_f$ are not visualized as they remain almost constant to zero throughout the task. As we can observe at the beginning of the task, the user starts with an ergonomic posture, i.e., $ a \approx 1 $. While the user tried to reach the line, her elbow angle $\theta_e$ increased to the risk value $\overline{\uptheta}_{eu}$, dropping the ergonomics factor to zero. During this phase (time window 8 - 10.8 sec), notable deviations between the robot end-effector and the god-object occur, which leads to increased resistance felt by the user. This resistance is evidenced by the corresponding increase in user-exerted force. It is important to note that while the user could perceive the resistance with less exerted force, in this proof-of-concept experiment, the subject was instructed to push the boundary intentionally. This was done to demonstrate the kinesthetic feedback mechanism inhibiting task continuation and to simulate a worst-case scenario. Under normal conditions, users should reconfigure their posture as soon as they perceive the PVFs. 
The user then adjusts her posture to a more ergonomic one, enabling her to reach the line with improved ergonomics. While gluing the line, her shoulder rotation angle $\theta_r$ exceeds its risk limit $\overline{\uptheta }_{ur}$ (time window 17.5 - 18.8 sec), resistance increases again, prompting another posture correction. After this adjustment, the user continues the task more ergonomically. 

In Figure \ref{fig:cam_res} (b), the damping terms related to the translational motion are only illustrated since the terms associated with rotational motion remain nearly constant, as the task predominantly involves translational motion. As observed, the damping term $d_{fp}$ associated with the ergonomics factor $a$ and with the deviation between the robot end-effector and the god-object increases as the deviation increases. For the other variable damping term $d_{vp}$, which varies based on the power transmitted from the user to the robot, it is observed that $d_{vp}$ takes its maximum value when $a=0$. This is because the motion in that case is almost zero, and therefore, the transmitted power is also minimal. The lack of motion can be explained by the fact that kinesthetic feedback is provided to the user, warning them of a non-ergonomic posture. Consequently, the user instinctively stops the end-effector's motion and adjusts his/her posture to adopt a more ergonomic one.

\begin{figure}[!t]
     \centering \includegraphics[width=0.48\textwidth]{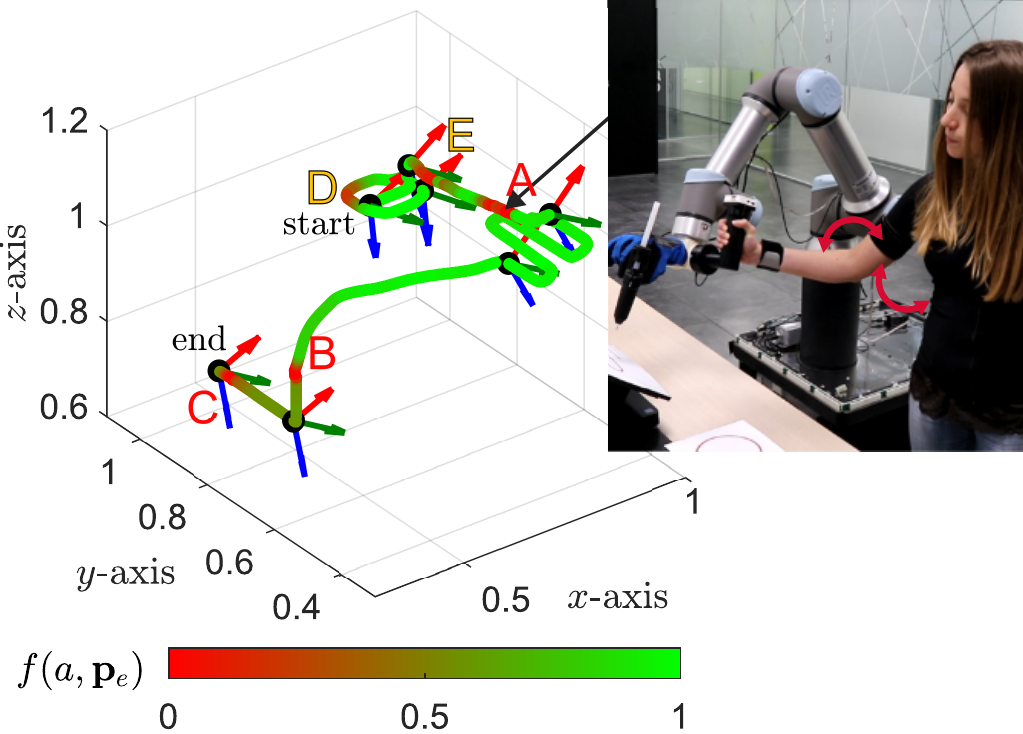}
     \vspace{-0.1cm}
    \caption{Visualization of the robot end-effector trace with color gradients representing $f(a,\mathbf p_e)$ values, alongside a time- stamped indication of a non-ergonomic period during the gluing task of the \textit{Prolonged Manipulation Experiment} (Experiment 2). Risk areas are labeled with A, B, and C, while D and E denote areas where at least one joint angle nears its risk limit without exceeding it.}
\label{fig:pathXSENS1}
\end{figure}

In the presented experiment, we demonstrated the functionality of the proposed control scheme  with a single expert user, showing that the system performs well using  a single RGB-D camera for skeleton tracking, provided it is not occluded and lighting conditions are adequate. However, this approach has inherent limitations, such as a restricted field of view and potential inaccuracies due to occlusions. To address these challenges, alternative approaches—such as employing multiple cameras or other monitoring devices—can be considered. In the subsequent experiments, the Xsens motion capture suit was utilized to overcome these limitations.

\begin{figure}[!t]
     \centering    
   { 
    \centering \includegraphics[width=0.45\textwidth]{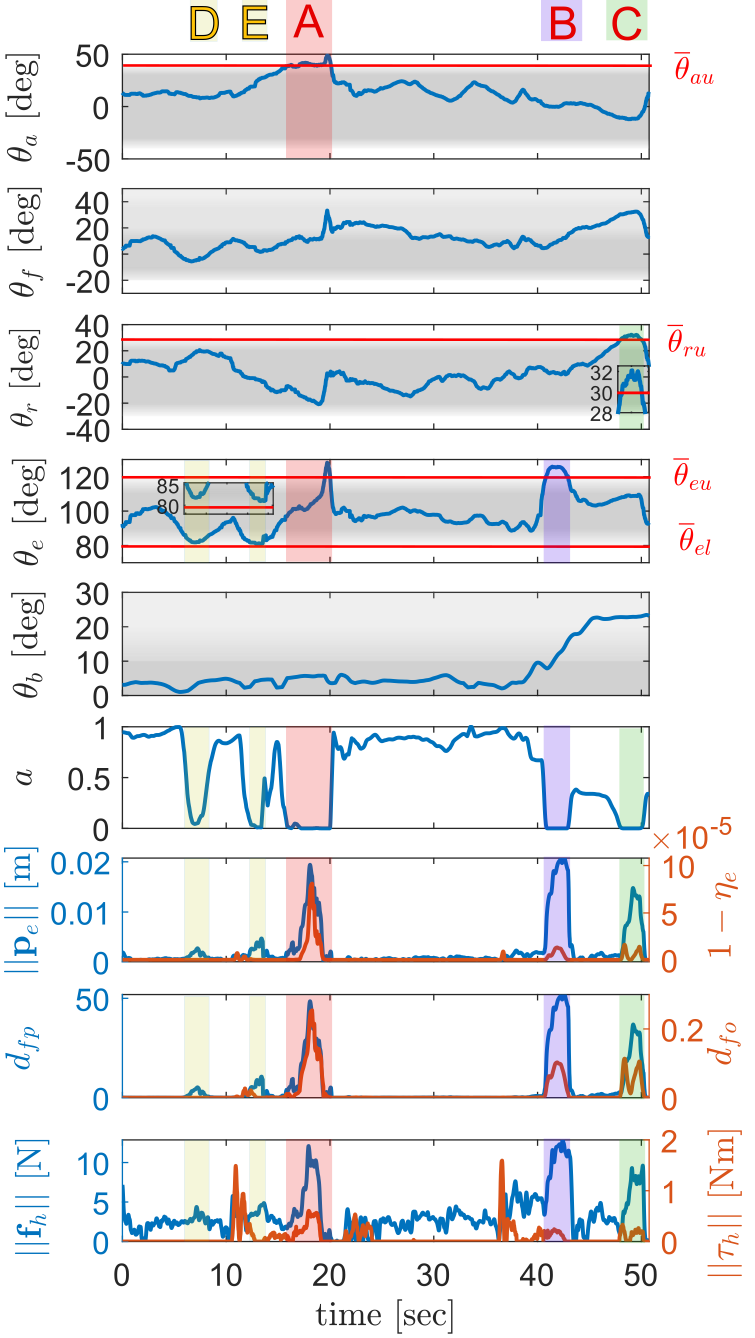}}
    \vspace{-0.3cm}
    \caption{From top to bottom, plots display human joint angles ($\theta_a, \theta_f, \theta_r, \theta_e, \theta_b$) with their corresponding ergonomic ranges, the ergonomics factor $a$, the distance metrics related to the robot pose from the god-object pose (translational $|| \mathbf p_e ||$ and rotational $1-\eta_e$), the ergonomics-based variable damping terms in the translational ($d_{fp}$) and rotational   ($d_{fo}$) motion, and the norms of human exerted force ($||\mathbf f_h||$) and torque ($||\tau_h||$) for the \textit{Prolonged Manipulation Experiment}, using IMU-based motion tracking system (Experiment 2).  The risk areas are labeled with A, B, and C, while D and E denote regions where at least one joint angle is near its risk limit without exceeding it.} 
\label{fig:xsens1_resa}
\end{figure}

\begin{figure}[!t]
     \centering    
    { \includegraphics[width=0.45\textwidth]{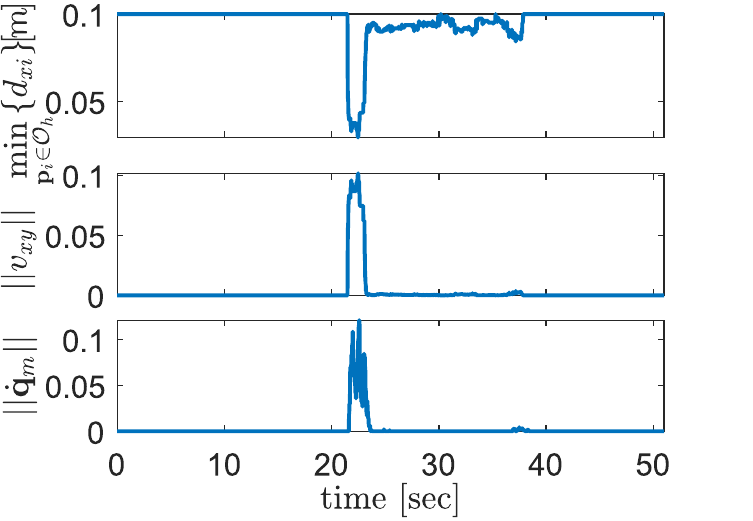}}
     \vspace{-0.3cm}
    \caption{Null-space adaptation data  for the \textit{Prolonged Manipulation} task, utilizing IMU-based motion tracking system (Experiment 2).}
\label{fig:xsens1_resb}
\end{figure}

\begin{figure*}[!t]
    \centering
    \subfloat[Human physical factors and ergonomics metrics]{ 
    \label{fig:exp1_statistical_analysis1}
        \includegraphics[width=0.48\linewidth]{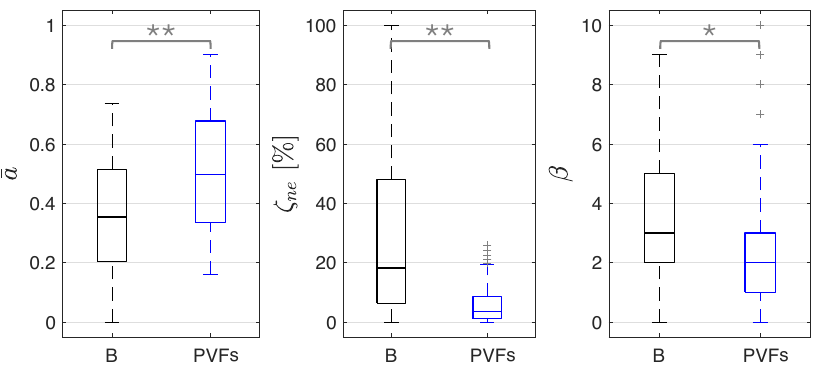}
    }
    \subfloat[User experience assessed through custom questionnaire]{ 
    \label{fig:exp1_statistical_analysis2}
        \includegraphics[width=0.4\linewidth]{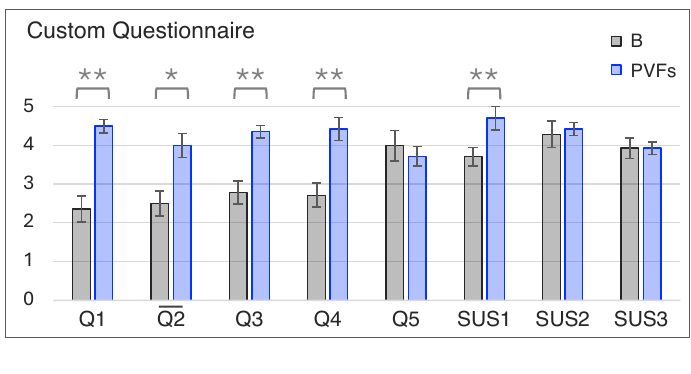}
    }
    \caption{
    Performance evaluation in the \textit{Prolonged Manipulation Task} (Experiment 2), comparing the proposed Postural Virtual Fixtures (PVFs) with the baseline condition (B) across all participants. Statistical significance levels are indicated at *$p<0.05$ and **$p<0.01$.
    }  \label{fig:exp1_statistical_analysis}
\end{figure*}
\subsection{Experiment 2: Prolonged Manipulation Task - IMUs Skeleton Tracking}
The objective of this experiment is to evaluate the proposed control framework in the prolonged manipulation task. An expert user first performs the task to demonstrate method functionality, and then the analysis is extended to multiple participants, each performing the task five consecutive times.
\subsubsection{Demonstration of Method Functionality:}  Figures \ref{fig:pathXSENS1}-\ref{fig:xsens1_resb} present the results of tracking the complete set of target paths in the \textit{Prolonged Manipulation Task}. As in the previous case, Figure \ref{fig:pathXSENS1} shows the trace of the robot end-effector, color-mapped according to the $f(a,\mathbf p_e)$ function, along with the end-effector frame indicating the orientation of the tool at the beginning and end of the task, as well as at the beginning and end of each tracking-target paths. 
Additionally, an image is shown where the human operator adopts a non-ergonomic posture due to pronounced shoulder abduction and elbow extension.

Figure \ref{fig:xsens1_resa} presents the angles of the human joints, the ergonomics factor, the pose deviation between the god-object and the robot end-effector, the ergonomics-based variable damping terms, and the human-exerted generalized force. As also observed in the previous experiment, when the angles enter the risk area, labeled A, B, and C, the pose deviation between the end effector and the god-object increases, causing the human operator to experience resistance, which serves as kinesthetic feedback for corrective human posture. This increase in pose deviation leads to a rise in the damping term \( d_{fp} \) as the translation deviation \( || \mathbf{p}_e || \) grows, and similarly, the damping term \( d_{fo} \) increases as the rotational deviation \( (1 - n_e) \) grows. When the angles are near the risk angle limits but have not yet exceeded them (labeled D and E), we can observe that the deviation between the end effector and the god-object increases slightly. However, this does not constrain their movement, as the angles have not yet exceeded the risk threshold, but the resistance is still noticeable. Notice that the highlighted areas A, B, C, D, and E are also connected with specific regions in the path plot shown in Figures \ref{fig:pathXSENS1}. The peaks in the torque around 11 seconds and 37 seconds are due to the applied torque to change the orientation of the tool, as can also be observed in the plot of the end effector frame in Figures \ref{fig:pathXSENS1}. 
In Figure \ref{fig:xsens1_resb}, the functionality of the null-space adaptation can be  observed when the user is in close proximity. In the top plot, we can see the minimum distance from the human point cloud to the platform, followed by the repulsive vector and the norm of the velocity platform. We can observe that as the user gets closer, the higher the repulsive vector, resulting in a higher velocity of the robotic platform, increasing the distance between the human and the platform. This demonstrates that the platform maintains a safe distance from the human's legs, providing her with enough space to complete the task.

\subsubsection{Multi-subject Validation:}

Figure \ref{fig:exp1_statistical_analysis}(a) presents the ergonomics metrics and corresponding statistical analysis for the two experimental conditions. A significant increase in the mean ergonomics factor $\bar{a}$ was observed when postural kinesthetic feedback was provided to the users, indicating that participants adopted and maintained a more ergonomic posture during the task with the proposed control framework. 
Additionally, the comparison with the baseline condition revealed a significant reduction in both the time spent in non-ergonomic configurations ($\zeta_{ne}$) and the frequency of crossing into the non-ergonomic posture region ($\beta$).
This demonstrates that participants effectively perceived the kinesthetic feedback, allowing them to quickly move away from non-ergonomic postures to complete the task.
Interestingly, the reduction in $\beta$ suggests that PVFs not only helps avoid non-ergonomic postures at the moment but also encourages users to avoid such postures and non-ergonomic boundaries in the future, thereby promoting ergonomics awareness.

\begin{figure}[!tb]
    \centering
    \includegraphics[width=0.77\linewidth]{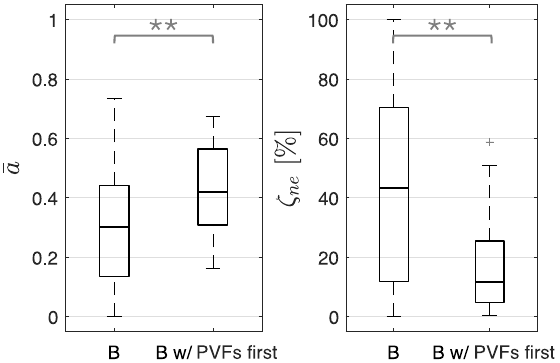}
    \caption{Effect of experimental condition order on ergonomic metrics in the \textit{Prolonged Manipulation Task} (Experiment~2), highlighting the learning to adopt a more ergonomic posture in participants who tested the Postural Virtual Fixtures (PVFs) condition first, even during the subsequent baseline condition (B) without kinesthetic feedback.}
    \label{fig:exp1_learning}
\end{figure}

\begin{figure*}[!t]
     \centering
     \includegraphics[width=\linewidth]{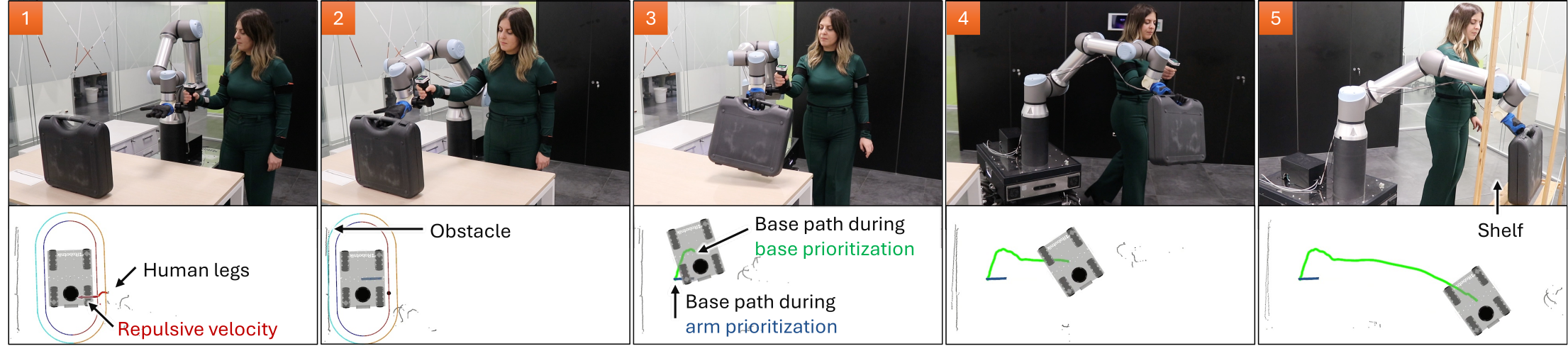}
    \caption{Excerpts of the main steps of the \textit{Long Distance Loco-Manipulation Task} (Experiment 3) with corresponding RViz visualizations: (1) As the user's legs approach the robotic base, a repulsive velocity is generated, prompting the base to move along the \textit{arm prioritization} path (blue) to avoid interfering with human movements; (2) the human-robot system grasps the box, with the robotic base halting to prevent collisions with obstacles; (3-4) the human-robot system transports the box to the drop-off location, switching to \textit{base prioritization} (green path) only when the user's posture is ergonomic; (5) the system places the box on the shelf.}
\label{fig:loco}
\end{figure*}
Participants' qualitative feedback, collected through the custom questionnaire for both experimental conditions, is shown in Figure \ref{fig:exp1_statistical_analysis}(b). The bars represent the mean values assigned by participants to the items, while the error bars indicate the $95\%$ confidence intervals of these means. For clarity, all results are reported such that higher values correspond to better outcomes. For items where the original scale indicated the opposite, a bar above the question number denotes that the scale has been inverted. 
The statistical analysis revealed that the PVFs significantly enhanced participants' perception of adopting and maintaining a more ergonomic posture during the task (\textit{Q1}), recognizing when they were approaching a non-ergonomic posture (\textit{Q2}), and the ease of reconfiguring their posture to complete tasks ergonomically (\textit{Q3}). 
Importantly, the kinesthetic feedback did not compromise participants' sense of control or confidence in using the system, as evidenced by no significant differences in \textit{Q5} ($p=0.44$) or \textit{SUS3} ($p=1.00$). Participants also rated the SRB as comparably easy to use (\textit{SUS2}, $p=0.77$) across both conditions, while they found the PVFs to be particularly well-integrated with the SRB's various functions (\textit{SUS1}, $p<0.01$).

Interestingly, the significant increase in values for \textit{Q4} across the two conditions suggests that participants felt the PVFs supported their learning to adopt and maintain ergonomic postures over repetitions. This perception was also confirmed by quantitative data. Regardless of the order in which participants performed the conditions, the PVFs significantly improved ergonomic metrics compared to the baseline. Howeover, participants who experienced the PVFs condition first demonstrated a significant improvement in the mean ergonomics factor (see the left plot of Figure \ref{fig:exp1_learning}) and reduction of time spent in non-ergonomic configurations (see the right plot of  Figure \ref{fig:exp1_learning}) during the subsequent baseline condition (without kinesthetic feedback), compared to those who started with the baseline condition.

NASA-TLX results showed a significant reduction in perceived physical demand with the PVFs feedback ($p<0.05$). Notably, the PVFs did not affect other scales, such as perceived mental demand, temporal demand, performance, effort, or frustration.
This finding is meaningful, as it highlights the effectiveness of the proposed framework in improving physical ergonomics without introducing additional cognitive burden, thereby maintaining overall task performance and user experience.

\subsection{Experiment 3: Long Distance Loco-Manipulation Task - IMUs Skeleton Tracking}
The objective of this experiment is to evaluate the proposed control framework in the long distance loco-manipulation task. First, the functionality is analyzed for an expert user, and then the analysis is extended to multiple participants.
\subsubsection{Demonstration of Method Functionality:}
Figures \ref{fig:loco}-\ref{fig:xsens2_res} show the results of executing the \textit{Long Distance Loco-Manipulation Task}. Figure \ref{fig:loco} captures key moments of the task with corresponding RViz visualizations, illustrating the designed capsule, the point cloud of the human legs, and the point cloud of the obstacle. The path of the platform's center is depicted in blue during \textit{arm prioritization} and in green during \textit{base prioritization}. The blue path highlights the platform's motion when the human is in close proximity to the base during manipulation facilitating the adoption of an ergonomic posture by the user.

\begin{figure}[!t]
    \centering 
    \includegraphics[width=0.45\textwidth]{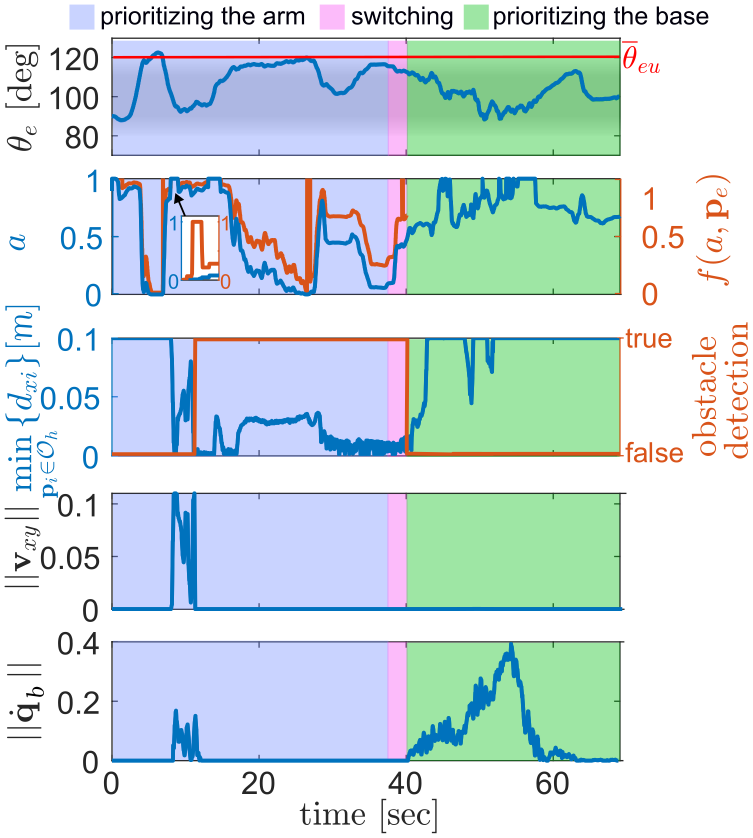}
    \vspace{-0.2cm}
    \caption{Results of the proposed controller for the \textit{Long Distance Loco-Manipulation Task} (Experiment 3).}
\label{fig:xsens2_res}
\end{figure}

\begin{figure}[!t]
    \centering
\includegraphics[width=0.8\linewidth]{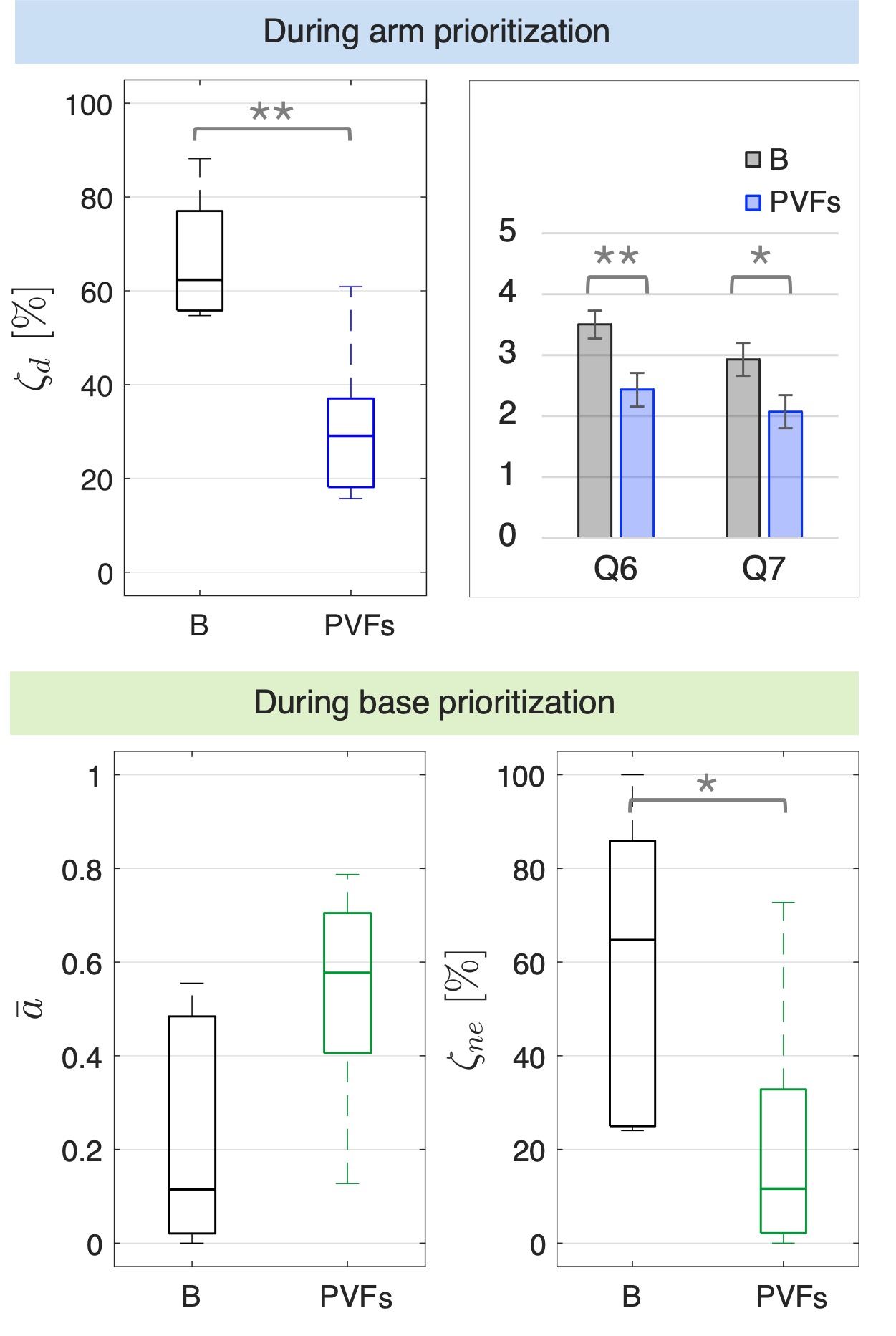}
    \caption{Performance evaluation in the \textit{Long Distance Loco-Manipulation Task} (Experiment 3). The top plots illustrate the quantitative and qualitative results demonstrating reduced robotic base interference with human movements during \textit{arm prioritization}. The bottom plots highlight the effect of improved ergonomics during \textit{base prioritization} as a result of the conditional switch that allows \textit{base prioritization} only when an ergonomic posture is assumed  during \textit{arm prioritization}.}
    \label{fig:exp2_statistical_analysis}
\end{figure}

In Figure \ref{fig:xsens2_res}, the phase prioritizing the robotic arm is highlighted with blue shading. Pink shaded regions show the attempt to switch priority to the base, which was temporarily delayed to allow the user to adopt a more ergonomic posture before transitioning to base prioritization for long-distance loco-manipulation. Green shaded areas indicate periods of active base prioritization. 
At the top, the elbow angle is shown as the only angle causing the ergonomics factor to drop to zero when exceeding its risk value around $t=5 \si{s}$.
In the subsequent plot, the ergonomics factor $a$ and the function  $\mathbf f(a,\mathbf p_e)$ are illustrated. With $m=0.5$ set in \eqref{eq:f}, the results yield  $f\geq a$, as observed also in the figure. Around $t=7 \si{s}$, we observe that the $\mathbf f(a,\mathbf p_e)=1$ while $a$ is almost zero. This occurred because, at that time, the user opted to utilize the framework's functionality to replan the movement by bringing her hand back before continuing the task, instead of reconfiguring her posture while maintaining the same hand pose. Specifically, the value of $\mathbf f(a,\mathbf p_e)$ enabled the god-object—and, consequently, the user's hand—to move in the direction of returning to the last ergonomic posture. 
The last three plots in the figure illustrate the functionality of the null-space adaptation. When priority is given to the manipulator and points are detected in the obstacle-related region, the adaptation of the robot's null space to support ergonomic posture adoption is disabled to prevent dangerous collisions with the environment, even if the user moves close to the platform. Consequently, the repulsive vector remains zero, and the base velocity is also zero. Notably, once the ergonomics factor $a$ reaches 0.5, the system automatically switches to base prioritization. What stands out in the plot is that $a$ maintains high values, never dropping to zero, even if kinesthetic feedback is not provided during this phase. In the present experiment, the obstacles in the environment are static; however, the proposed framework also supports time-variable obstacles, as shown in the attached video.

\subsubsection{Multi-subject Validation:}
Across all subjects, as was expected, the proposed control framework, which accounts for the proximity of the user to the robotic base, led to a significant reduction in the percentage of time the robotic base remained within the critical proximity threshold compared to the baseline, as shown in the top-left plot of Figure \ref{fig:exp2_statistical_analysis}.   This is also reflected in the custom questionnaire, where participants reported less interference (\textit{Q6}, top-right plot) and fewer concerns about collisions with surrounding obstacles (\textit{Q7},  top-left plot). Furthermore, this base adaptation, combined with PVF feedback, resulted in significant improvement in ergonomic metrics compared to the baseline ($p<0.01$), indicating that the proposed control framework  enabled participants to adopt ergonomic postures without interference from the robotic base during precise manipulation.  

Notably, the carryover effect induced by conditional switching, which allowed \textit{base prioritization} only when the user assumed  an ergonomic posture,  confirmed our hypothesis that promoting ergonomics during \textit{arm prioritization} would positively influence posture during  \textit{base prioritization}. Specifically, the mean ergonomics factor $\bar{a}$ during \textit{base prioritization} increased by $136\%$ (marginal significant, $p=0.08$), accompanied by a significant reduction in the time spent in non-ergonomic postures, even in the absence of kinesthetic feedback during \textit{base prioritization} (see the bottom plots of Figure \ref{fig:exp2_statistical_analysis}). 
The results of other questionnaire items (\textit{Q1}-\textit{Q5} and NASA-TLX measures across both task phases, i.e., \textit{arm prioritization} and \textit{base prioritization}) aligned with findings from the previous experiment. These outcomes confirm that the proposed framework enhances both ergonomics and user experience in human-robot collaborative loco-manipulation tasks.

\subsection{Discussion}
The experimental results highlight the functionality and effectiveness of the proposed control scheme in providing kinesthetic feedback to enhance posture awareness and promote ergonomic practices among SRB users. Users were observed to adopt more ergonomic postures when provided with kinesthetic feedback, spending less time in ergonomically risky postures and engaging in such postures less frequently. This suggests that the feedback immediately positively impacted posture, reducing the time spent in potentially harmful configurations.
Furthermore, the comparison of ergonomics metrics across the two condition orders (with and without feedback) provides initial evidence that the feedback may have a learning effect, as users exhibited improved ergonomics even after the feedback was removed, compared to those who had not previously experienced the proposed control framework. This indicates that the feedback not only promoted short-term improvements but also helped users establish better posture habits that persisted beyond the active intervention.

In addition to improving posture, the results revealed enhanced human-robot coordination, with users spending less time in close proximity to the robotic platform, suggesting that the platform does not hinder human movement. A particularly notable finding is that enabling the robotic base prioritization in the SRB controller only when users achieved an ergonomic posture above a defined threshold promoted better posture maintenance throughout the task, reducing the likelihood of users adopting risky postures during locomotion. This suggests that starting with a good posture can have a lasting effect on maintaining better ergonomics during locomotion tasks involving long transportation, as users tend to maintain the same arm posture throughout these tasks.
Quality results using questionnaires further support the quantitative findings.

Importantly, the proposed control framework is versatile, as it is not limited to specific robotic devices. It can be applied to any system that accepts position or velocity commands. 
Furthermore, testing with both vision-based and IMU-based setups shows that a specific motion-tracking approach is not required, as long as the critical joint positions of the user are monitored.     
This flexibility suggests that the framework has the potential for widespread application in pHRI tasks, both within SRB setups and in broader contexts.

\section{CONCLUSIONS} 

This study introduced a novel control framework for SRBs designed to promote ergonomic posture through kinesthetic feedback, utilizing virtual fixture techniques integrated with a continuous and online ergonomic assessment framework. Additionally, the floating base of the SRB could dynamically adjust based on the operator's proximity, ensuring that the platform avoids collisions with obstacles in the environment and further enhances coordination between the robot and the operator. To evaluate the effectiveness of the proposed framework, a user study involving fourteen participants was conducted. The study included two distinct tasks: the first task focused on fine motion, prioritizing the robotic arm of the SRB, while the second task combined fine and larger motions, requiring prioritization of both the arm and the base. Statistical analysis of both quantitative and qualitative metrics demonstrated that the proposed SRB solution effectively helps users maintain a more ergonomic posture, with additional evidence of a learning effect, suggesting that the benefits of the framework extend beyond immediate usage. A key limitation of the current framework is that the online ergonomic assessment focuses exclusively on the upper body, even though, in the  pHRI tasks, the lower limbs (i.e., hips, thighs, knees, ankles, and feet) are also potentially susceptible to strain and discomfort. Additionally, IMU sensor drift may impact motion tracking, and the 2D LiDAR representation limits the effectiveness of collision avoidance. To address these issues, future research will incorporate lower body joint angles, use sensor fusion to reduce drift, and explore 3D LiDAR or sensor combinations for better environmental mapping. The current framework focuses on ergonomic assessment based on human kinematics, but does not account for forces. Future work will also integrate the dynamics of human movement, such as forces, torques, and energy exchanges. Another planned improvement is shifting from an IMU-based motion-tracking suit to a multi-camera setup for human skeleton detection. This would provide more flexibility by eliminating the need for a wearable device and avoiding the drift associated with IMU-based tracking.

\section*{APPENDIX}

\subsection{Joint Motion Analysis}
{Given the  shoulder $\mathbf{S}_i$, elbow $\mathbf{E}_i$, wrist $\mathbf{W}_i$, neck $\mathbf{Ne}$,  middle of the thorax $\mathbf{Th}$, middle of the pelvis $\mathbf{Pl}$, and knee $\mathbf{Kn}$ keypoints, where $i=\mathrm{\{R,L\}}$ indicates the right or left side of the human body, respectively, we can first define the frontal/coronal plane and the sagittal/lateral plane, described by the normal vectors $\mathbf f_i$ and $\mathbf s_i$, respectively, as follows:
\begin{equation}
\mathbf f_i =\dfrac{ \overrightarrow{\mathbf{Ne}\mathbf{S}_i }\times \overrightarrow{\mathbf{NeTh}} }{\|\overrightarrow{\mathbf{Ne}\mathbf{S}_i }\times \overrightarrow{\mathbf{NeTh}}\|} \text{ and } \mathbf s_i=\dfrac{\overrightarrow{\mathbf{Ne S}_i}}{\|\overrightarrow{\mathbf{Ne S}_i}\|}.
\end{equation}
Note that the vectors are orthogonal, i.e., $\mathbf f_i^\mathrm{T}\mathbf s_i=0$.}

{The shoulder movement can be characterized by three angles:  
abduction/adduction  $ \theta^{i}_a$, flexion/extension $ \theta^{i}_{f}$, and internal/external rotation $ \theta^{i}_{r}$. The shoulder abduction/adduction angle $\theta^{i}_a$ is defined as the angle between the upper arm and trunk in the frontal plane, and it can be measured based on the angle between the vectors  $\overrightarrow{\mathbf{Ne S}_i} $ and $\overrightarrow{\mathbf{S}_i \mathbf{E}_i} $ projected in the coronal plane (see Figure \ref{fig:human_angle}(a)) as follows:
\begin{equation}
    \theta^{i}_a = \begin{cases}
        - \cos^{-\mathrm{1}}\left(-\dfrac{{{\overrightarrow{\mathbf{S}_i \mathbf{E}_i} }^{'\mathrm{T}}} {\mathbf{s}_i }}{\|{\overrightarrow{\mathbf{S}_i \mathbf{E}_i}^{'} }\| }\right) + \dfrac{3\pi}{2},&  \text{if } ( {r_a^{i}}  < 0\ \land  \\ & \hspace{-0.3cm}{\overrightarrow{\mathbf{S}_i \mathbf{E}_i} }^{'\mathrm{T}}\mathbf{s}_i>  0 )  \vspace{0.2cm}  \\
       {r_a^{i}}  \cos^{-\mathrm{1}}\left(-\dfrac{{{\overrightarrow{\mathbf{S}_i \mathbf{E}_i} }^{'\mathrm{T}}} {\mathbf{s}_i }}{\|{\overrightarrow{\mathbf{S}_i \mathbf{E}_i}^{'} }\| }\right)  - \dfrac{\pi}{2}, & \text{otherwise}
    \end{cases}
\end{equation}
where ${r_a^{i}} = \mathrm{sign}\big(\mathbf{f}_i^\mathrm{T}(\mathbf s_i \times{\overrightarrow{\mathbf{S}_i \mathbf{E}_i} }^{'}  )\big)   $ and
${\overrightarrow{\mathbf{S}_i \mathbf{E}_i} }^{'} =(\mathbf I-\mathbf f_i \mathbf f_i^\mathrm{T}){\overrightarrow{\mathbf{S}_i \mathbf{E}_i} } $}.

{The shoulder flexion/extension angle $\theta^{i}_{f}$ is measured within the sagittal plane.  As presented in Figure \ref{fig:human_angle}(b), it can be calculated using the triangles shown in the figure by determining the cosine of the angles between the normal unit vectors of each corresponding triangle as follows  \citep{miyashita2008relationship}:
\begin{equation}
    \theta^{i}_{f} =
       {r_f^{i}}  \cos^{-\mathrm{1}}({{{\mathbf{f}_i} }^\mathrm{T} \mathbf{u}_i  }) 
\end{equation}
where $ \mathbf{u}_i=\dfrac{\overrightarrow{ \mathbf{Ne}\mathbf{S}_i} \times \overrightarrow{\mathbf{NeE}_i} }{\| \overrightarrow{ \mathbf{Ne}\mathbf{S}_i} \times \overrightarrow{\mathbf{NeE}_i} \|} $,  ${r_f^{\mathrm R}} = \mathrm{sign}\big(\mathbf{s}_{\mathrm R}^\mathrm{T}(\mathbf{f}_{\mathrm R}\times \mathbf{u}_{\mathrm R})\big) \vspace{0.5cm}   $ and ${r_f^{\mathrm L}} = \mathrm{sign}\big(-\mathbf{s}_{\mathrm L}^\mathrm{T}(\mathbf{f}_{\mathrm L}\times \mathbf{u}_{\mathrm L})\big) $.}  

{The shoulder internal/ external angle $\theta^{i}_{r}$ is calculated  using the triangles shown in Figure \ref{fig:human_angle}(c) as follows:
\begin{equation}
    \theta^{i}_{r} = \begin{cases}
        -  \cos^{-\mathrm{1}}\left({\mathbf{u}_i }^\mathrm{T}\mathbf{v}_i\right) + \dfrac{3\pi}{2},&  \text{if } ({r_l^{i}}  < 0 \ \land  {\mathbf{u}_i }^\mathrm{T}\mathbf{v}_i<  0)  \vspace{0.2cm}  \\
       r_l^{i}\cos^{-\mathrm{1}}\left({\mathbf{u}_i }^\mathrm{T}\mathbf{v}_i\right)-\dfrac{\pi}{2}, & \text{otherwise}
    \end{cases}
\end{equation}where $ \mathbf{v}_i=\dfrac{\overrightarrow{\mathbf{E}_i\mathbf{S}_i}\times \overrightarrow{ \mathbf{E}_i\mathbf{W}_i} }{\|\overrightarrow{\mathbf{E}_i\mathbf{S}_i}\times \overrightarrow{ \mathbf{E}_i\mathbf{W}_i} \|} $,   ${r_l^{\mathrm R}} = \mathrm{sign}\big({\overrightarrow{\mathbf{E}_{\mathrm R}\mathbf{S}_{\mathrm R}}}^\mathrm{T}(\mathbf{u}_{\mathrm R}\times \mathbf{v}_{\mathrm R})\big)   $ and  ${r_l^{\mathrm L}} = \mathrm{sign}\big({\overrightarrow{\mathbf{E}_{\mathrm L}\mathbf{S}_{\mathrm L}}}^\mathrm{T}(\mathbf{v}_{\mathrm L}\times \mathbf{u}_{\mathrm L})\big)   $ .}

{The elbow flexion/extension angle $\theta^{i}_{e}$ can be calculated based on the angle between the vectors $\overrightarrow{ \mathbf{E}_i\mathbf{W}_i}$ and $\overrightarrow{\mathbf{E}_i\mathbf{S}_i}$ (see Figure \ref{fig:human_angle}(d)) as follows:
\begin{equation}
    \theta^{i}_{e} = \cos^{-\mathrm{1}}\bigg(\dfrac{\overrightarrow{ \mathbf{E}_i\mathbf{W}_i}^\mathrm{T}\overrightarrow{\mathbf{E}_i\mathbf{S}_i}}{\| \overrightarrow{ \mathbf{E}_i\mathbf{W}_i}\| 
\|\overrightarrow{\mathbf{E}_i\mathbf{S}_i}\|}\bigg) 
\end{equation}}

{The bending angle can be calculated based on the angle between the vectors ${\overrightarrow{ \mathbf{KnPl}}}$ and $ \overrightarrow{\mathbf{PlNe}} $ (see Figure \ref{fig:human_angle}(e)) as follows:
\begin{equation}
    \theta_{b} =r_b \cos^{-\mathrm{1}}\bigg(\dfrac{{\overrightarrow{ \mathbf{KnPl}}}^\mathrm{T}\overrightarrow{\mathbf{PlNe}}}{\| {\overrightarrow{ \mathbf{KnPl}}}\| 
\|\overrightarrow{\mathbf{PlNe}}\|}\bigg) 
\end{equation}
where  ${r_b} = \mathrm{sign}\big({\overrightarrow{\mathbf{Ne}\mathbf{S_R}}}^\mathrm{T}({\overrightarrow{ \mathbf{PlNe}}}\times {\overrightarrow{\mathbf{KnPl} }})\big)   $.
\begin{remark}
To ensure consistency in the results, the previously calculated angle is used when the vector norms involved in the divisions are zero.
\end{remark}}
\subsection{DLS for Whole-Body Inverse Kinematics}
Consider the Jacobian $\mathbf{J}_b\in \mathbb{R}^{6 \times 3}$, which maps the floating base joint velocities to the generalized velocities of the end-effector, and $\mathbf{J}_a \in \mathbb{R}^{6 \times n}$, the Jacobian associated with the manipulator's joints. The complete robot body Jacobian is then given by:
$\mathbf{J} = [\mathbf{J}_b \ \ \mathbf{J}_a]\in \mathbb{R}^{6\times(3+n)}$.

To solve the inverse kinematics problem, we employ the Damped Least-Squares (DLS) method, originally proposed in \cite{nakamura1986inverse} and later utilized by \cite{giammarino2024super} for SRBs. This approach introduces a damping factor to regularize the solution, ensuring smooth and stable joint movements while maintaining robustness in singular configurations. The solution to the damped least-squares problem is given by:
\begin{equation}
\begin{split}
    \dot{\mathbf{ q}}_p&=(\mathbf J^\mathrm{T} \mathbf W_x\mathbf J + \mathbf W_q)^{-1} \mathbf J^\mathrm{T} \mathbf W_x( \mathbf v-\mathbf K_x \mathbf x_e)\\&=\mathbf J^* (\mathbf v-\mathbf K_x \mathbf x_e) \in \mathbb{R}^{(3+n)} 
\end{split}
\end{equation}
where $\mathbf v \in \mathbb R ^6$  is the generalized desired velocity of the robot end effector, which is calculated by performing the numerical integration of the admittance model \eqref{eq:adm_x2y},  $\mathbf x_e  \in \mathbb R ^6$  represents the pose error between the admittance model pose and the actual end-effector pose,  $\mathbf K_x \in \mathbb{R}^{6 \times 6}$ is the pose-error feedback gain matrix that helps minimize deviations caused by the introduction of the joint velocity damping matrix $\mathbf W_q \in \mathbb{R}^{(3+n) \times (3+n)}$ and $\mathbf J^*\in \mathbb{R}^{(3+n)\times 6}$ is the so-called Singularity Robust Inverse. The weighting matrices $\mathbf W_x \in \mathbb{R}^{6 \times 6}$ and $\mathbf W_q $ are positive defined, thus the matrix $\mathbf J^\mathrm{T} \mathbf W_x\mathbf J + \mathbf W_q$ is positive definite and hence nonsingular, i.e., invertible. As in \cite{giammarino2024super}, the matrix $\mathbf{W}_x$ is designed to introduce priority into the task vector, while the weight matrix $\mathbf{W}_q$ can be adjusted to prioritize specific joints by assigning higher weight values to those with lower priority, increasing their damping relative to others, and it can also normalize the effects of different joint types by adjusting their weights accordingly. Notice that for redundancy resolution, the null-space projection matrix
$\mathcal{N}=(\mathbf I-\mathbf J^\mathrm{*} \mathbf J)$, is utilized to project a vector onto the null space of the Jacobian.
\subsection{Closest Point on a 2D Capsule} Let us consider a capsule in 2D space  defined by a line segment with length $\mathrm{L} \in \mathbb{R}_{\geq 0}$ and radius $\mathrm{r}_c \in \mathbb{R}_{\geq 0}$. The axis of the capsule is described by \cite{kastritsi2022passive}:
\begin{equation}
\mathbf{p}_l(\sigma) = \mathbf{p}_f + \mathbf{u}_{rf} \mathrm{L} \sigma, \quad \sigma \in [0, 1],   
\end{equation}
where $\mathbf{p}_f \in \mathbb{R}^2$ is the front point of the line and $\mathbf{u}_{rf} \in \mathbb{R}^2$  is the unit vector pointing in the direction of the capsule's length.

For each point  $\mathbf{p}_i \in \mathbb{R}^2$, 
closest point on the capsule's perimeter is computed by:
\begin{equation}
\mathbf{p}_i^c = \mathbf{p}_i^* + \frac{\mathbf{p}_i - \mathbf{p}_i^*}{\|\mathbf{p}_i - \mathbf{p}_i^*\|} \mathrm{r}_c
\end{equation}
where $\mathbf{p}_i^*=\mathbf{p}_i(\sigma^*)$ 
is the nearest point on the capsule's line with $\sigma^*$ given by :
\begin{equation}
\sigma_i^* = \begin{cases} 
   \zeta_i & \text{if } 0 \leq \zeta_i \leq 1, \\
   1 & \text{if } \zeta_i > 1, \\
   0 & \text{if } \zeta_i < 0,
   \end{cases} \quad \zeta_i = \frac{1}{\mathrm{L}} \mathbf{u}_{rf}^\mathrm{T} (\mathbf{p}_i - \mathbf{p}_f).
\end{equation}

\section*{Acknowledgments}
We would like to express our sincere appreciation to Dr. Francisco Jesus Ruiz for the design of the robot handler and his technical support in the 3D printing process.
\section*{Declaration of conflicting interests}
The author(s) declared no potential conflicts of interest with respect to the research, authorship, and/or publication of this article.
\section*{Funding}
This work was supported by the Italian Ministry of University and Research (MUR) under the Fondo Italiano per la Scienza (FIS), call FIS 3, project EPIC (code FIS-2024-02654), and the Italian National Institute for Insurance against Accidents at Work (INAIL) ergoCub Core Project. 
\section{ORCID iDs}
Theodora Kastritsi 
\href{https://orcid.org/0000-0002-5379-0259}{https://orcid.org/0000-0002-5379-0259} \\[2mm]
Marta Lagomarsino \href{https://orcid.org/0000-0001-9121-1812}{https://orcid.org/0000-0001-9121-1812} \\[2mm] 
Arash Ajoudani \href{https://orcid.org/0000-0002-1261-737X}{https://orcid.org/0000-0002-1261-737X} 

\bibliographystyle{sageh}
\bibliography{ref}

@article{hignett2000rapid,
  title={Rapid entire body assessment (REBA)},
  author={Hignett, Sue and McAtamney, Lynn},
  journal={Applied ergonomics},
  volume={31},
  number={2},
  pages={201--205},
  year={2000},
  publisher={Elsevier}
}

@book{Khalil_book,
	author        = "H. Khalil",
	title         = "Nonlinear Systems",
	publisher     = "Prentice Hall",
	edition       = "Third",
	year          = "2002"
}

@article{kermavnar2021effects,
  title={Effects of industrial back-support exoskeletons on body loading and user experience: An updated systematic review},
  author={Kermavnar, Tja{\v{s}}a and de Vries, Aijse W and de Looze, Michiel P and O’Sullivan, Leonard W},
  journal={Ergonomics},
  volume={64},
  number={6},
  pages={685--711},
  year={2021},
  publisher={Taylor \& Francis}
}

@inproceedings{koutras2021exponential,
  title={Exponential stability of trajectory tracking control in the orientation space utilizing unit quaternions},
  author={Koutras, Leonidas and Doulgeri, Zoe},
  booktitle={2021 IEEE/RSJ International Conference on Intelligent Robots and Systems (IROS)},
  pages={8151--8158},
  year={2021},
  organization={IEEE}
}

@article{bar2021influence,
  title={The influence of using exoskeletons during occupational tasks on acute physical stress and strain compared to no exoskeleton--A systematic review and meta-analysis},
  author={B{\"a}r, Mona and Steinhilber, Benjamin and Rieger, Monika A and Luger, Tessy},
  journal={Applied Ergonomics},
  volume={94},
  pages={103385},
  year={2021},
  publisher={Elsevier}
}

@article{sirintuna2024enhancing,
  title={Enhancing human--robot collaborative transportation through obstacle-aware vibrotactile warning and virtual fixtures},
  author={Sirintuna, Doganay and Kastritsi, Theodora and Ozdamar, Idil and Gandarias, Juan M and Ajoudani, Arash},
  journal={Robotics and Autonomous Systems},
  volume={178},
  pages={104725},
  year={2024},
  publisher={Elsevier}
}

@inproceedings{kim2020moca,
  title={MOCA-MAN: A MObile and reconfigurable Collaborative Robot Assistant for conjoined huMAN-robot actions},
  author={Kim, Wansoo and Balatti, Pietro and Lamon, Edoardo and Ajoudani, Arash},
  booktitle={2020 IEEE international conference on robotics and automation (ICRA)},
  pages={10191--10197},
  year={2020},
  organization={IEEE}
}

@inproceedings{sidiropoulos2021variable,
  title={A variable admittance controller for human-robot manipulation of large inertia objects},
  author={Sidiropoulos, Antonis and Kastritsi, Theodora and Papageorgiou, Dimitrios and Doulgeri, Zoe},
  booktitle={2021 30th IEEE International Conference on Robot \& Human Interactive Communication (RO-MAN)},
  pages={509--514},
  year={2021},
  organization={IEEE}
}

@article{miyashita2008relationship,
  title={Relationship between maximum shoulder external rotation angle during throwing and physical variables},
  author={Miyashita, Koji and Urabe, Yukio and Kobayashi, Hirokazu and Yokoe, Kiyoshi and Koshida, Sentaro and Kawamura, Morio and Ida, Kunio},
  journal={Journal of Sports Science \& Medicine},
  volume={7},
  number={1},
  pages={47},
  year={2008},
  publisher={Dept. of Sports Medicine, Medical Faculty of Uludag University}
}

@inproceedings{lecours2012variable,
  title={Variable admittance control of a four-degree-of-freedom intelligent assist device},
  author={Lecours, Alexandre and Mayer-St-Onge, Boris and Gosselin, Cl{\'e}ment},
  booktitle={2012 IEEE international conference on robotics and automation},
  pages={3903--3908},
  year={2012},
  organization={IEEE}
}

@inproceedings{lamy2009achieving,
  title={Achieving efficient and stable comanipulation through adaptation to changes in human arm impedance},
  author={Lamy, Xavier and Colledani, Fr{\'e}d{\'e}ric and Geffard, Franck and Measson, Yvan and Morel, Guillaume},
  booktitle={2009 IEEE international conference on Robotics and automation},
  pages={265--271},
  year={2009},
  organization={IEEE}
}

@article{rosenberg1992use,
  title={The use of virtual fixtures as perceptual overlays to enhance operator performance in remote environments},
  author={Rosenberg, Louis B},
  journal={Air force material command},
  pages={1--42},
  year={1992}
}

@article{keemink2018admittance,
  title={Admittance control for physical human--robot interaction},
  author={Keemink, Arvid QL and Van der Kooij, Herman and Stienen, Arno HA},
  journal={The International Journal of Robotics Research},
  volume={37},
  number={11},
  pages={1421--1444},
  year={2018},
  publisher={SAGE Publications Sage UK: London, England}
}

@article{yang2021supernumerary,
  title={Supernumerary robotic limbs: A review and future outlook},
  author={Yang, Bo and Huang, Jian and Chen, Xinxing and Xiong, Caihua and Hasegawa, Yasuhisa},
  journal={IEEE Transactions on Medical Robotics and Bionics},
  volume={3},
  number={3},
  pages={623--639},
  year={2021},
  publisher={IEEE}
}

@article{tong2021review,
  title={Review of research and development of supernumerary robotic limbs},
  author={Tong, Yuchuang and Liu, Jinguo},
  journal={IEEE/CAA Journal of Automatica Sinica},
  volume={8},
  number={5},
  pages={929--952},
  year={2021},
  publisher={IEEE/CAA Journal of Automatica Sinica}
}

@article{de2016exoskeletons,
  title={Exoskeletons for industrial application and their potential effects on physical work load},
  author={De Looze, Michiel P and Bosch, Tim and Krause, Frank and Stadler, Konrad S and O’sullivan, Leonard W},
  journal={Ergonomics},
  volume={59},
  number={5},
  pages={671--681},
  year={2016},
  publisher={Taylor \& Francis}
}

@article{theurel2019occupational,
  title={Occupational exoskeletons: overview of their benefits and limitations in preventing work-related musculoskeletal disorders},
  author={Theurel, Jean and Desbrosses, Kevin},
  journal={IISE Transactions on Occupational Ergonomics and Human Factors},
  volume={7},
  number={3-4},
  pages={264--280},
  year={2019},
  publisher={Taylor \& Francis}
}

@article{mcatamney1993rula,
  title={RULA: a survey method for the investigation of work-related upper limb disorders},
  author={McAtamney, Lynn and Corlett, E Nigel},
  journal={Applied ergonomics},
  volume={24},
  number={2},
  pages={91--99},
  year={1993},
  publisher={Elsevier}
}

@inproceedings{zilles1995constraint,
  title={A constraint-based god-object method for haptic display},
  author={Zilles, Craig B and Salisbury, J Kenneth},
  booktitle={Proceedings 1995 ieee/rsj international conference on intelligent robots and systems. Human robot interaction and cooperative robots},
  volume={3},
  pages={146--151},
  year={1995},
  organization={IEEE}
}

@inproceedings{ruspini1997haptic,
  title={The haptic display of complex graphical environments},
  author={Ruspini, Diego C and Kolarov, Krasimir and Khatib, Oussama},
  booktitle={Proceedings of the 24th annual conference on Computer graphics and interactive techniques},
  pages={345--352},
  year={1997}
}

@article{stavridis2017dynamical,
  title={Dynamical system based robotic motion generation with obstacle avoidance},
  author={Stavridis, Sotiris and Papageorgiou, Dimitrios and Doulgeri, Zoe},
  journal={IEEE Robotics and Automation Letters},
  volume={2},
  number={2},
  pages={712--718},
  year={2017},
  publisher={IEEE}
}

@article{nakamura1986inverse,
    author = {Nakamura, Yoshihiko and Hanafusa, Hideo},
    title = {Inverse Kinematic Solutions With Singularity Robustness for Robot Manipulator Control},
    journal = {Journal of Dynamic Systems, Measurement, and Control},
    volume = {108},
    number = {3},
    pages = {163-171},
    year = {1986},
    month = {09},
    issn = {0022-0434},
    doi = {10.1115/1.3143764},
    url = {https://doi.org/10.1115/1.3143764},
    eprint = {https://asmedigitalcollection.asme.org/dynamicsystems/article-pdf/108/3/163/5482965/163_1.pdf},
}

@article{hasanen2024design,
  title={Design of twisted string actuated flexure joint for supernumerary robotic arm for bi-manual tasks},
  author={Hasanen, Basma and Suthar, Bhivraj and Zweiri, Yahya and Seneviratne, Lakmal and Hussain, Irfan},
  journal={IEEE Sensors Journal},
  year={2024},
  publisher={IEEE}
}

@inproceedings{morfino2024hybrid,
  title={A Hybrid Position/Force Control for Robot-Aided Pedicle Tapping in Spinal Surgery},
  author={Morfino, Rosaura and Lauretti, Clemente and Cordella, Francesca and Zollo, Loredana},
  booktitle={2024 10th IEEE RAS/EMBS International Conference for Biomedical Robotics and Biomechatronics (BioRob)},
  pages={1004--1010},
  year={2024},
  organization={IEEE}
}

@inproceedings{hendriks2024enhancing,
  title={Enhancing Functional and Extra Motor Abilities: A Focus Group Study on the Re-Design of an Extra-Robotic Finger},
  author={Hendriks, Sjoerd and Hasanen, Basma and Afzal, Naqash and Hussain, Irfan and Obaid, Mohammad},
  booktitle={2024 33rd IEEE International Conference on Robot and Human Interactive Communication (ROMAN)},
  pages={667--673},
  year={2024},
  organization={IEEE}
}

@article{prattichizzo2021human,
  title={Human augmentation by wearable supernumerary robotic limbs: review and perspectives},
  author={Prattichizzo, Domenico and Pozzi, Maria and Baldi, Tommaso Lisini and Malvezzi, Monica and Hussain, Irfan and Rossi, Simone and Salvietti, Gionata},
  journal={Progress in Biomedical Engineering},
  volume={3},
  number={4},
  pages={042005},
  year={2021},
  publisher={IOP Publishing}
}

@inproceedings{balatti2024robot,
  title={Robot-assisted navigation for visually impaired through adaptive impedance and path planning},
  author={Balatti, Pietro and Ozdamar, Idil and Sirintuna, Doganay and Fortini, Luca and Leonori, Mattia and Gandarias, Juan M and Ajoudani, Arash},
  booktitle={2024 IEEE International Conference on Robotics and Automation (ICRA)},
  pages={2310--2316},
  year={2024},
  organization={IEEE}
}

@article{benzi2022whole,
  title={Whole-body control of a mobile manipulator for passive collaborative transportation},
  author={Benzi, Federico and Mancus, Cristian and Secchi, Cristian},
  journal={IFAC-PapersOnLine},
  volume={55},
  number={38},
  pages={106--112},
  year={2022},
  publisher={Elsevier}
}

@inproceedings{navarro2017framework,
  title={A framework for intuitive collaboration with a mobile manipulator},
  author={Navarro, Benjamin and Cherubini, Andrea and Fonte, A{\"\i}cha and Poisson, G{\'e}rard and Fraisse, Philippe},
  booktitle={2017 IEEE/RSJ International Conference on Intelligent Robots and Systems (IROS)},
  pages={6293--6298},
  year={2017},
  organization={IEEE}
}

@article{hu2024proxy,
  title={Proxy-based guidance virtual fixtures with orientation constraints},
  author={Hu, Weitao and Pan, Xinan and Wang, Hongguang},
  journal={International Journal of Intelligent Robotics and Applications},
  pages={1--11},
  year={2024},
  publisher={Springer}
}

@article{giammarino2024super,
  title={SUPER-MAN: SUPERnumerary robotic bodies for physical assistance in huMAN--robot conjoined actions},
  author={Giammarino, Alberto and Gandarias, Juan M and Balatti, Pietro and Leonori, Mattia and Lorenzini, Marta and Ajoudani, Arash},
  journal={Mechatronics},
  volume={103},
  pages={103240},
  year={2024},
  publisher={Elsevier}
}

@article{lorenzini2023ergonomic,
  title={Ergonomic human-robot collaboration in industry: A review},
  author={Lorenzini, Marta and Lagomarsino, Marta and Fortini, Luca and Gholami, Soheil and Ajoudani, Arash},
  journal={Frontiers in Robotics and AI},
  volume={9},
  pages={813907},
  year={2023},
  publisher={Frontiers}
}

@article{proia2021control,
  title={Control techniques for safe, ergonomic, and efficient human-robot collaboration in the digital industry: A survey},
  author={Proia, Silvia and Carli, Raffaele and Cavone, Graziana and Dotoli, Mariagrazia},
  journal={IEEE Transactions on Automation Science and Engineering},
  volume={19},
  number={3},
  pages={1798--1819},
  year={2021},
  publisher={IEEE}
}

@article{kastritsi2022passive,
  title={A passive admittance controller to enforce remote center of motion and tool spatial constraints with application in hands-on surgical procedures},
  author={Kastritsi, Theodora and Doulgeri, Zoe},
  journal={Robotics and Autonomous Systems},
  volume={152},
  pages={104073},
  year={2022},
  publisher={Elsevier}
}

@article{liao2023ergo,
  title={An Ergo-Interactive Framework for Human-Robot Collaboration Via Learning From Demonstration},
  author={Liao, Zhiwei and Lorenzini, Marta and Leonori, Mattia and Zhao, Fei and Jiang, Gedong and Ajoudani, Arash},
  journal={IEEE Robotics and Automation Letters},
  year={2023},
  publisher={IEEE}
}

@inproceedings{leigh2015person,
  title={Person tracking and following with 2d laser scanners},
  author={Leigh, Angus and Pineau, Joelle and Olmedo, Nicolas and Zhang, Hong},
  booktitle={2015 IEEE international conference on robotics and automation (ICRA)},
  pages={726--733},
  year={2015},
  organization={IEEE}
}

@article{bowyer2013active,
  title={Active constraints/virtual fixtures: A survey},
  author={Bowyer, Stuart A and Davies, Brian L and y Baena, Ferdinando Rodriguez},
  journal={IEEE Transactions on Robotics},
  volume={30},
  number={1},
  pages={138--157},
  year={2013},
  publisher={IEEE}
}

@inproceedings{kastritsi2019stability,
  title={Stability of active constraints enforcement in sensitive regions defined by point-clouds for robotic surgical procedures},
  author={Kastritsi, Theodora and Papageorgiou, Dimitrios and Sarantopoulos, Iason and Doulgeri, Zoe and Rovithakis, George A},
  booktitle={2019 18th European Control Conference (ECC)},
  pages={1604--1609},
  year={2019},
  organization={IEEE}
}

@article{papageorgiou2020passive,
  title={A passive robot controller aiding human coaching for kinematic behavior modifications},
  author={Papageorgiou, Dimitrios and Kastritsi, Theodora and Doulgeri, Zoe},
  journal={Robotics and Computer-Integrated Manufacturing},
  volume={61},
  pages={101824},
  year={2020},
  publisher={Elsevier}
}

@article{papageorgiou2020passive2,
  title={A passive phri controller for assisting the user in partially known tasks},
  author={Papageorgiou, Dimitrios and Kastritsi, Theodora and Doulgeri, Zoe and Rovithakis, George A},
  journal={IEEE Transactions on Robotics},
  volume={36},
  number={3},
  pages={802--815},
  year={2020},
  publisher={IEEE}
}

@article{castillo2010virtual,
  title={Virtual fixtures with autonomous error compensation for human--robot cooperative tasks},
  author={Castillo-Cruces, Raul A and Wahrburg, J{\"u}rgen},
  journal={Robotica},
  volume={28},
  number={2},
  pages={267--277},
  year={2010},
  publisher={Cambridge University Press}
}

@article{bowyer2015dissipative,
  title={Dissipative control for physical human--robot interaction},
  author={Bowyer, Stuart A and y Baena, Ferdinando Rodriguez},
  journal={IEEE Transactions on Robotics},
  volume={31},
  number={6},
  pages={1281--1293},
  year={2015},
  publisher={IEEE}
}

@inproceedings{ryden2011proxy,
  title={Proxy method for fast haptic rendering from time varying point clouds},
  author={Ryden, Fredrik and Kosari, Sina Nia and Chizeck, Howard Jay},
  booktitle={2011 IEEE/RSJ International Conference on Intelligent Robots and Systems},
  pages={2614--2619},
  year={2011},
  organization={IEEE}
}

@article{ferraguti2020unified,
  title={A unified architecture for physical and ergonomic human--robot collaboration},
  author={Ferraguti, Federica and Villa, Renzo and Landi, Chiara Talignani and Zanchettin, Andrea Maria and Rocco, Paolo and Secchi, Cristian},
  journal={Robotica},
  volume={38},
  number={4},
  pages={669--683},
  year={2020},
  publisher={Cambridge University Press}
}

@inproceedings{shafti2019real,
  title={Real-time robot-assisted ergonomics},
  author={Shafti, Ali and Ataka, Ahmad and Lazpita, B Urbistondo and Shiva, Ali and Wurdemann, Helge A and Althoefer, Kaspar},
  booktitle={2019 International Conference on Robotics and Automation (ICRA)},
  pages={1975--1981},
  year={2019},
  organization={IEEE}
}

@inproceedings{zanchettin2019collaborative,
  title={Collaborative robot assistant for the ergonomic manipulation of cumbersome objects},
  author={Zanchettin, Andrea Maria and Lotano, Elio and Rocco, Paolo},
  booktitle={2019 IEEE/RSJ International Conference on Intelligent Robots and Systems (IROS)},
  pages={6729--6734},
  year={2019},
  organization={IEEE}
}

@article{hart2006nasa,
author = {Sandra G. Hart},
title ={Nasa-Task Load Index (NASA-TLX); 20 Years Later},

journal = {Proceedings of the Human Factors and Ergonomics Society Annual Meeting},
volume = {50},
number = {9},
pages = {904-908},
year = {2006},
doi = {10.1177/154193120605000909},

URL = { 
    
        https://doi.org/10.1177/154193120605000909
    
    

},
eprint = { 
    
        https://doi.org/10.1177/154193120605000909
    
    

}
}

@article{lewis2018system,
  title={The system usability scale: past, present, and future},
  author={Lewis, James R},
  journal={International Journal of Human--Computer Interaction},
  volume={34},
  number={7},
  pages={577--590},
  year={2018},
  publisher={Taylor \& Francis}
}

@article{lagomarsino2023maximising,
  title={Maximising coefficiency of human-robot handovers through reinforcement learning},
  author={Lagomarsino, Marta and Lorenzini, Marta and Constable, Merryn Dale and De Momi, Elena and Becchio, Cristina and Ajoudani, Arash},
  journal={IEEE Robotics and Automation Letters},
  volume={8},
  number={8},
  pages={4378--4385},
  year={2023},
  publisher={IEEE}
}

@article{lorenzini2022online,
  title={An online multi-index approach to human ergonomics assessment in the workplace},
  author={Lorenzini, Marta and Kim, Wansoo and Ajoudani, Arash},
  journal={IEEE Transactions on Human-Machine Systems},
  volume={52},
  number={5},
  pages={812--823},
  year={2022},
  publisher={IEEE}
}

@article{dominijanni2023human,
  title={Human motor augmentation with an extra robotic arm without functional interference},
  author={Dominijanni, Giulia and Pinheiro, Daniel Leal and Pollina, Leonardo and Orset, Bastien and Gini, Martina and Anselmino, Eugenio and Pierella, Camilla and Olivier, J{\'e}r{\'e}my and Shokur, Solaiman and Micera, Silvestro},
  journal={Science Robotics},
  volume={8},
  number={85},
  pages={eadh1438},
  year={2023},
  publisher={American Association for the Advancement of Science}
}

@inproceedings{raei2024multipurpose,
  title={A multipurpose interface for close-and far-proximity control of mobile collaborative robots},
  author={Raei, Hamidreza and Gandarias, Juan M and De Momi, Elena and Balatti, Pietro and Ajoudani, Arash},
  booktitle={2024 10th IEEE RAS/EMBS International Conference for Biomedical Robotics and Biomechatronics (BioRob)},
  pages={457--464},
  year={2024},
  organization={IEEE}
}

@article{eden2022principles,
  title={Principles of human movement augmentation and the challenges in making it a reality},
  author={Eden, Jonathan and Br{\"a}cklein, Mario and Ib{\'a}{\~n}ez, Jaime and Barsakcioglu, Deren Yusuf and Di Pino, Giovanni and Farina, Dario and Burdet, Etienne and Mehring, Carsten},
  journal={Nature Communications},
  volume={13},
  number={1},
  pages={1345},
  year={2022},
  publisher={Nature Publishing Group UK London}
}

\end{document}